\documentclass[10pt,a4paper]{article}

\usepackage[utf8]{inputenc}
\usepackage[T1]{fontenc}
\usepackage{lmodern}
\usepackage{microtype}
\usepackage[margin=1in]{geometry}

\usepackage{amsmath,amssymb,amsthm}
\usepackage{mathtools}
\usepackage{tikz}
\usetikzlibrary{bayesnet} 
\usepackage{soul}
\usepackage{booktabs}

\usetikzlibrary{arrows.meta,positioning}

\usepackage{xcolor}

\usepackage{enumitem}
\usepackage{parskip}
\usepackage{xparse}

\usepackage{hyperref}
\hypersetup{
    colorlinks=true,
    linkcolor=blue!70!black,
    citecolor=green!50!black,
    urlcolor=red!70!black,
}
\usepackage{cleveref}

\definecolor{theoremcolor}{RGB}{230, 240, 255}
\definecolor{definitioncolor}{RGB}{255, 248, 230}
\definecolor{examplecolor}{RGB}{240, 255, 240}
\definecolor{remarkcolor}{RGB}{255, 240, 245}

\usepackage[most]{tcolorbox}
\tcbuselibrary{theorems,skins,breakable}

\NewDocumentCommand{\DeclareTheoremBox}{O{} O{} m m m m O{3pt}}{%
  \newtcbtheorem[#1]{#3}{#4}{%
    enhanced, breakable,
    colback=#5, colbacktitle=#5,
    boxrule=0pt, frame hidden,
    arc=1.6mm,
    left=4mm, right=5mm, top=2mm, bottom=2mm,
    underlay unbroken={\begin{tcbclipinterior}
      \fill[#6] (interior.north west)
        rectangle ([xshift=#7]interior.south west);
    \end{tcbclipinterior}},
    underlay first={\begin{tcbclipinterior}
      \fill[#6] (interior.north west)
        rectangle ([xshift=#7]interior.south west);
    \end{tcbclipinterior}},
    underlay middle={\begin{tcbclipinterior}
      \fill[#6] (interior.north west)
        rectangle ([xshift=#7]interior.south west);
    \end{tcbclipinterior}},
    underlay last={\begin{tcbclipinterior}
      \fill[#6] (interior.north west)
        rectangle ([xshift=#7]interior.south west);
    \end{tcbclipinterior}},
    attach title to upper,
    before upper={\setlength{\parskip}{0.5em}\tcbtitle\quad},
    fonttitle=\bfseries, coltitle=black,
    separator sign={:\,},
    description delimiters={(}{)},
    #2 %
  }{#3}%
}

\definecolor{theoBack}{RGB}{235,245,255}  %
\definecolor{theoLine}{RGB}{56,128,255}   %
\definecolor{defiBack}{RGB}{234,245,233}  %
\definecolor{defiLine}{RGB}{48,205,47}   %
\definecolor{lemBack}{RGB}{244,242,251}  %
\definecolor{lemLine}{RGB}{200,183,229}   %
\definecolor{exampleBack}{RGB}{231,246,250}  %
\definecolor{exampleLine}{RGB}{118,209,225}   %
\definecolor{noteBack}{RGB}{245,245,248}  %
\definecolor{noteLine}{RGB}{194,194,200}   %
\definecolor{claimBack}{RGB}{227,222,238}  %
\definecolor{claimLine}{RGB}{138,108,199}   %

\DeclareTheoremBox[number within=section][label type=theorem]{theorem}{Theorem}{theoBack}{theoLine}
\DeclareTheoremBox[use counter from=theorem][label type=lemma]{lemma}{Lemma}{lemBack}{lemLine}
\DeclareTheoremBox[use counter from=theorem][label type=corollary]{corollary}{Corollary}{theoBack}{theoLine}
\DeclareTheoremBox[use counter from=theorem][label type=proposition]{proposition}{Proposition}{theoBack}{theoLine}
\DeclareTheoremBox[use counter from=theorem][label type=definition]{definition}{Definition}{defiBack}{defiLine}
\DeclareTheoremBox[use counter from=theorem][label type=example]{example}{Example}{exampleBack}{exampleLine}
\DeclareTheoremBox[use counter from=theorem][label type=remark]{remark}{Note}{noteBack}{noteLine}
\DeclareTheoremBox[use counter from=theorem][label type=claim]{claim}{Claim}{claimBack}{claimLine}

\title{\Large\bfseries AI Agents as Universal Task Solvers: It's All About Time
}
\author{Alessandro Achille \and Stefano Soatto}
\date{
\vspace{-0.6em}
AWS Agentic AI\\
\vspace{1.0em}
September 12, 2025
}

\begin{document}

\maketitle

\begin{abstract}
We describe AI agents as stochastic dynamical systems and frame the problem of \textit{learning to reason} as in \textit{transductive inference}: Rather than approximating the distribution of past data as in classical induction, the objective is to capture its \textit{algorithmic structure} so as to reduce the time needed to solve new tasks. In this view, information from past experience serves not only to reduce a model's uncertainty --- as in Shannon's classical theory --- but to reduce the computational effort required to find solutions to unforeseen tasks.
Working in the verifiable setting, where a checker or reward function is available,
we establish three main results. First, we show that the optimal speed-up on a new task is tightly related to the algorithmic information it shares with the training data, yielding a theoretical justification for the power-law scaling empirically observed in reasoning models. Second, while the \textit{compression view} of learning, rooted in Occam's Razor, favors simplicity, we show that transductive inference yields its greatest benefits precisely when the data-generating mechanism is most complex. Third, we identify a possible failure mode of na\"ive scaling: in the limit of unbounded model size and compute, models with access to a reward signal can behave as \textit{savants} --- brute-forcing solutions without acquiring transferable reasoning strategies. Accordingly, we argue that a critical quantity to optimize when scaling reasoning models is \textit{time}, whose role in learning has remained largely unexplored.

\end{abstract}

\section{Introduction}

Much of machine learning focuses on \textit{inductive inference}: fitting a function to past data and expecting it to generalize to similar inputs. While valuable, this perspective is incomplete. In an agentic setting, we want a pre-trained model to \textit{reason} at inference time about the structure of a \textit{specific instance of a novel task}, rather than simply apply a map fixed during training. From this alternative viewpoint, the role of learning shifts from building a fixed classifier to equipping a solver with the knowledge needed to efficiently tackle previously unseen problems. This perspective draws on what Vapnik termed \textit{transductive inference} \cite{vapnik2006transductive} --- reasoning directly about specific instances rather than applying a mapping fixed at training time.

Many tasks in machine learning can be --- and historically have been --- approached inductively: classifying an object, for example, relies mainly on similarity with past observations rather than reasoning about the instance itself. However, AI is increasingly tackling problems where this framework falls short. Consider writing code that passes a given set of unit tests, proving a theorem with access to a proof checker, or finding a protein configuration that minimizes an energy function. In these cases, the task is not implicitly defined by past examples, as in inductive learning, but rather --- at least approximately --- by a known \textit{verifier} or \textit{reward signal} capable of evaluating any candidate solution. More precisely, each problem instance $x$ is paired with a function $f(x, y)$ that can score or verify any candidate solution $y$.

This setting poses a fundamentally different challenge from inductive learning. Finding a correct solution is, in principle, always \textit{achievable}: one could enumerate candidates $y$ until one satisfies $f(x,y)$. This brute-force strategy guarantees success, provided we are willing to wait a time exponential in the length of the solution. Unlike in inductive learning, the core difficulty is therefore not accuracy or generalization -- which can always be ensured given enough time -- but \textit{computational efficiency}: quickly finding a solution to a previously unseen task $f$ on an instance $x$.

For any fixed task, one could of course design or train a \textit{bespoke solver}, carefully optimized for that specific problem, that finds solutions extremely fast. But could there exist a \textit{general} solver that tackles any unseen task nearly as efficiently as the best task-specific solver?

That would seem too good to be true, and likely to violate some kind of ``no free lunch theorem.'' Yet, Levin \cite{levin1973universal} and  Solomonoff \cite{solomonoff1984optimum} showed that a \textit{universal solver} $U$ can solve any instance $x$ of a task \textit{essentially} as well as a solver $A$ that is optimal for that task. In particular, the universal solver finds a solution in time $T_U$ bounded by:
\[
T_U(x) \;\le\; 2^{\ell(A)}\, T_A(x),
\]
where $\ell(A)$ is the description length of $A$ and $T_A(x)$ is the optimal time to find a solution to that instance. Crucially, the universal solver requires only a constant factor longer than a task-specific solver, where the constant depends on the \textit{complexity of the optimal solver}, not on the particular instance $x$. The catch is that such constant factor $2^{\ell(A)}$ can be astronomically large.

This is where \textit{learning} comes in: In \cite{solomonoff1984optimum} Solomonoff observed that, even if a task has never been faced before, prior experience lets us encode effective problem-solving programs more succinctly --- e.g., by reusing components of the solution --- thereby reducing the factor $2^{\ell(A)}$. Thus, in the transductive setting with no uncertainty on the reward, the value of learning is measured \textit{not} by a reduction in error rate as in induction, but in the reduction of the \textit{time} it takes to find solutions to unforeseen tasks. This points to a foundational principle for transductive learning: Rather than trying to capture the statistical structure (joint distribution) of past data in hope that future data will respect it, as in induction, in transduction we want to capture the shared algorithmic structure of past data, given which an agent can reason to quickly find solutions to \textit{new} computable tasks. Accordingly, to be able to solve \textit{general} unseen tasks an agent should \textit{learn to perform transductive inference}, that is, they should ``learn to reason''.

Learning transductive inference --- as envisioned by Solomonoff in 1985 --- has been mostly an academic concern for decades, because it seems to require (meta-)learning a conditional distributions over \textit{programs} that, when run through a Universal Turing Machine (UTM), solve a given task. This has been impractical until recently. But modern large-scale models, such as large language models (LLMs) or reasoning models, can effectively encode distributions over programs. Moreover, they themselves serve, loosely speaking, as a powerful new type of \textit{computation engines} --- planning, searching, calling tools, and coordinating multi-step reasoning --- which however fundamentally differs from UTMs and does not cleanly fit neither the theory of Solomonoff and Levin, nor the classical mold of inductive learning. 

In this work, we embrace the transductive view and formalize learning in a way that fits modern reasoning models. We show in what sense an LLM, despite being stochastic and in many ways antithetical to Turing Machines, is a valid model of universal computation. We then contextualize Levin's guarantees and Solomonoff's vision using LLMs instead of UTMs --- which requires entirely different proof techniques --- study resource-aware objectives, and ground these ideas into practical algorithms that trade inference-time compute to generalize to novel problem not seen during training.

Our analysis focuses on universal solvers in a verifiable setting, where an oracle is provided to compute the reward or verify a solution. This restriction yields a clean theoretical framework, but it is admittedly narrow: most real-world problems involve both an inductive component --- reducing uncertainty from data --- and a transductive component --- reasoning efficiently about a specific instance. Classical statistical learning theory addresses the former; our theory aims to address the latter. We do not tackle the general case in which the two are intertwined --- where reasoning is required and the reward, or an efficient approximation of it, must itself be learned. Nevertheless, we believe that the basic mechanisms we identify --- enabling a model to become a universally efficient solver in the spirit of Solomonoff --- and the tradeoffs they reveal among learning, algorithmic information, and time, extend beyond the purely verifiable setting.

\subsection{Can an LLM-powered AI Agent be a universal solver?}

Levin and Solomonoff showed that a universal solver exists, but the construction hinges on using (deterministic) Universal Turing Machines (UTMs). LLMs, by contrast, are neither Turing Machines nor deterministic, nor do they execute code in the conventional sense. Their \textit{computational mechanism} is chain-of-thought reasoning (CoT), which does not map easily to any standard computational paradigm.

To study whether an LLM-powered AI agent\footnote{By \textit{agent}, we generally refer to an LLM performing Chain-of-Thought in order to solve a task, potentially with access to tools or an environment. For a given model and hardware, the compute cost of the agent roughly corresponds to the number of CoT tokens generated to find a solution.} can be a universal solver, we need more flexible foundations. In Section~\ref{sec:dynamical-systems-as-computation}, we extend universal solvers from programs to general stochastic dynamical systems, allowing us to map the theory directly to LLMs with CoT. A key challenge is defining the \textit{time that an LLM-powered agent needs to find a solution to a task}: naively using the expected length of CoT leads to degenerate values. We address this by introducing a new notion of \textit{proper time}~$\tau$ (\Cref{definition:proper-time}). This notion allows us to generalize Levin's and Solomonoff's results to general dynamical systems, showing in particular that LLMs can indeed power universal task solvers despite being unlike any Turing Machine.

\subsection{Intelligence is about time}

Once we have secured the foundations for reasoning and computation, we turn to learning. Universal solvers of verifiable tasks are peculiar in that \textit{no information needs to be learned to achieve universally optimal accuracy on any task}. For instance, to prove a theorem one could simply iterate over all possible proofs until a correct one appears, without ever having to study any math. Indeed, if we were to measure the information that a training dataset provides to such a task using classical notions (Shannon's \cite{cover1991information} or Kolmogorov's \cite{li2008introduction}), we would find it to be null \cite{achille2019information}.

The role of learning in agentic AI is instead to identify statistical or algorithmic structures that make future inference \textit{more time-efficient}. This suggests a notion, complementary to Shannon's, that \textbf{time plays a crucial role in learning information}. In particular, we show that the optimal speed-up in finding a solution that a universal solver can achieve using past data is tightly related to the \textit{algorithmic mutual information} \cite{li2008introduction} between the data and the solution:

\begin{theorem}{Information is Speed}{info-speed}
The maximum speed-up a task universal solver can achieve in finding an optimal solution $h$ to a task from training on a dataset $D$ is
\[
\log \text{speed-up} = I(h : D)
\]
where $I(h : D)$ is the algorithmic mutual information between the data and the solution.
\end{theorem}
We will give detailed definitions and proofs in the next sections; for now, we call attention to the fact that data can make a solver \textit{exponentially faster}, consistent with our view that learning transduction can be equated to \textit{amortizing inference computation}.

\subsection{Scaling laws for speed}
Having established that past data can speed up universal solvers, we examine how much speed-up is achievable as a function of the training dataset size. This requires making modeling assumptions about the data generating process or underlying mechanism.

A common assumption, including in Solomonoff's work, is that real data, while complex on the surface, is generated through mechanisms or phenomena with low intrinsic complexity (Occam's Razor \cite{blumer1987occam}). This would intuitively suggest that there are `common reusable components' we can learn from past data to help future reasoning. This intuition is, however, severely misleading, since the maximum speed-up obtainable by a solver is bounded by the complexity of the data generating distribution:

\begin{theorem}{Maximum Speed-up Bound}{max-speed-up-intro}
The maximum speed-up in an optimal solution $h$ of a task sampled from a data generating process $q$ --- from which the training dataset $D \sim q$ is sampled ---  is
\[
\log \text{speed-up} \le K(q)
\]
where the Kolmogorov complexity $K(q)$ is the length of the shortest program for $q$.
\end{theorem}

Notably, if the data was generated by the Universal Prior (as in Solomonoff Induction \cite{solomonoff1964formal1,solomonoff1964formal2}), there would be precisely nothing to learn (zero information). This challenges the fundamental assumption in generalization theory that \textit{simplicity aids learning} \cite{catoni2003pac}. While simplicity benefits \textit{explainability}, it does not necessarily improve learning effectiveness. That simpler models generalize better is a consequence of how generalization is defined in the inductive setting; recent developments in LLMs have shown that simpler models are generally less effective at transduction.

From the results above, we see that the effectiveness of learning --- and the asymptotic validity of scaling laws --- hinges on the data generation process having effectively unbounded complexity. If complexity were bounded at $K(q)$, scaling laws would plateau there; yet empirically we observe non-saturating power-law scaling \cite{kaplan2020scaling,snell2024scaling}. This characteristic power-law trend is captured by Hilberg's law for human-generated data \cite{hilberg1990bekannte,debowski2015relaxed}:

\begin{definition}{Hilberg Scaling}{hilberg}
Let $X_n$ be a training dataset of $n$ tokens and $Y_n$ be a test set of $n$ tokens, then:
\[
I(X_n : Y_n) \propto n^\beta
\]
grows unbounded according to some distribution-specific rate $\beta \in (0, 1)$.
\end{definition}

We introduce a generalization of this conjecture for arbitrarily-sized $X_n$ and $Y_m$, and prove the following \textbf{scaling law for speed after learning from data}:

\begin{theorem}{Power Law of Inference vs Training}{power-law}
Let $h$ be a chain-of-thought trajectory solving a task, and let $T_h$ be its length. If the generalized Hilberg's law holds, the log-speed-up from training on $n$ tokens is
\[
\log \text{speed-up}(n) = T_h^\beta - \beta \frac{T_h}{n^{1-\beta}}
\]
\end{theorem}

This result provides a \emph{strong theoretical justification} for the empirically observed power-law scaling of inference time versus training time in reasoning LLMs \cite{snell2024scaling}, and can also be used to predict the scaling of space-bound LLMs (when the number of weights, rather than data, is the limit), thus providing guidance on how to scale resources when training universal solvers.

\subsection{Inversion of scaling laws}
The results so far dealt with the best model that \textit{could} be learned from the data as we scale up. Empirically, models follow predicted power-law trends, suggesting near-optimal learning. But is this necessarily true? Is bigger always better?

Current scaling laws using prediction error (or perplexity) are sometime used a proxy for intelligence, arguing that more data, bigger models and more compute will lead to ``super-intelligence'' \cite{aschenbrenner2024situational}. But, counterintuitively, as models become more powerful, learning becomes unnecessary since the model can rely more on exhaustive computation rather than insights from learned structure in the data. As ordinary scaling proceeds, better and better performance on benchmarks may come with less and less insight, all the way to the limit where infinite resources allow solving any task by brute force without any learning. More precisely, emergence of ``intelligence'' (in the sense of \textit{intus legere}, ``reading into'' or insight) goes hand-in-hand with optimizing a solution under time constraints. Informally:

\begin{theorem}{Learning and Time Optimization}{learning-time}
Without time penalties, optimal inference can be achieved brute-force without any learning or insight. Conversely, any system that optimizes time must learn at least $I(h : D) = \log \text{speed-up}$ bits from past data.
\end{theorem}

The results above reveal that plots of accuracy-versus-size, routinely used to predict progress towards `super-intelligence,' can be misleading. By ignoring the cost of time, they encourage \textit{savantry} over \textit{intelligence}. Intelligent behavior should instead be measured by \textit{success rate per unit time/compute}.
Properly accounting for the cost of time using the net reward, we see that an optimal agent balances time and accuracy rather than blindly maximizing reward through brute-force search. This optimal operating point depends on the task, the environment, and the particular user's needs. Hence, it cannot be predetermined during training, but rather must be calibrated and controlled by the user at deployment to fit the situation at hand.

The theorems above highlight a core tension: in inductive learning, achieving generalization relies on \textit{minimizing} the complexity of the trained model --- that is, the amount of information it shares with the specific training set --- in keeping with uniform generalization bounds and Occam's Razor. In universal solvers, by contrast, we seek to \textit{maximize} the information the trained model shares with the training data. Such memorization is beneficial because facts that were irrelevant to the training task may turn out to be handy when reasoning to solve future tasks.

\subsection{Maximizing a Continuous Reward}

Thus far we have focused on systems that need to reach a verifiable goal. More generally, a system may be trying to optimize a continuous reward function $R(x, y)$. This introduces two challenges: we do not know when we have found the maximum, and we do not know if the increase in reward from further exploration will compensate the cost (which also depends on the user). Note that, while in reinforcement learning the solver is tasked with finding a policy for a particular environment (which may be partly unknown), here we seek solvers that can operate in \textit{any} environment, each with its own reward; when the reward function itself is unknown, we must again face the problem of learning it inductively.

If we have no way to estimate the predicted improvement in reward as a function of additional effort, tree search algorithms are worst-case optimal \cite{orseau2018single}.
If, on the other hand, we can forecast future rewards \cite{kleinman2025e1}, we show that the problem can be formalized as a Pandora's box problem \cite{weitzman1978optimal,boodaghians2020pandora}, which provides both an optimal search strategy and an optimal stopping rule.

In the next sections, we expand on the summary above in greater detail.

\section{Dynamical Systems as Computation}
\label{sec:dynamical-systems-as-computation}

Let us consider an example of an LLM solving a task using Chain-of-Thought (CoT):

\begin{example}{}{}
\setlength{\parskip}{0.5em}
\textbf{What is the largest zero of $x^3 + 5 x^2 - 74 x + 120$?}

{I need to find the zeros of the polynomial $f(x) = x^3 + 5x^2 - 74x + 120$. Let me try to factor this polynomial. First, I'll check if there are any rational roots using the rational root theorem. 

The possible rational roots are the factors of 120 divided by the factors of 1, which means the possible rational roots are: $\pm 1, \pm 2, \pm 3, \pm 4, \pm 5, \pm 6, \pm 8, \pm 10, \pm 12, \pm 15, \pm 20, \pm 24, \pm 30, \pm 40, \pm 60, \pm 120$.

Let me test some of these:
For $x = 1: 1 + 5 - 74 + 120 = 52 \neq 0$
For $x = -1: -1 + 5 + 74 + 120 = 198 \neq 0$
For $x = 2: 8+20-148+120=0$ \checkmark

So $x = 2$ is a root! This means $(x - 2)$ is a factor.

Now I'll perform polynomial division to find the other factor:
$x^3 + 5x^2 - 74x + 120 = (x - 2)(x^2 + a x + b)$

Expanding the right side:
$(x - 2)(x^2 + a x + b) = x^3 + a x^2 + b x - 2x^2 - 2a x - 2b = x^3 + (a-2)x^2 + (b-2a)x - 2b$

Comparing coefficients:
\begin{itemize}
\item Coefficient of $x^2$: $a - 2 = 5$, so $a = 7$
\item Coefficient of $x$: $b - 2a = -74$, so $b - 14 = -74$, thus $b=-60$
\end{itemize}

I'll continue factoring $x^2 + 7x - 60$ by finding two numbers that multiply to $-60$ and add to 7. The numbers 12 and $-5$ work perfectly, since $12 \times (-5) = -60$ and $12 + (-5) = 7$.
This means I can split the quadratic into $(x + 12)(x - 5)$.

Combining all factors gives $x^3 + 5x^2 - 74x + 120 = (x - 2)(x + 12)(x - 5)$, revealing the zeros at $x = 2$, $x = -12$, and $x = 5$. The largest zero is 5.
}
\end{example}

Clearly, this chain of thought is acting as some kind of computational device. However, it is nothing resembling a program since the results of the computation are intermixed with instructions (code) and comments, e.g.,
\[
\text{For } x=1, \underbrace{1 + 5 - 74 + 120}_{\text{code}} = \underbrace{52}_{\text{result}}
\]
It is also not the trace of execution of an underlying algorithm, since the code is generated dynamically based on the output of previous operations. And, unlike standard programs, the particular CoT tokens are not necessarily meaningful: just outputting dots ``$\ldots$'' may still lead to the correct result \cite{pfau2024let}.

A more sound view is that an autoregressive LLM perfroming CoT is a stochastic dynamical system, where the prompt represents the initial state, the trained model backbone the transition probability, and each CoT trace is a sample trajectory. CoT \textit{performs computation} in the sense that, starting from an initial state, the system evolves until it reaches a \textit{terminating} state (one where the network is confident it can answer), at which point it outputs a final answer.

Of course, it is well understood that a (deterministic or stochastic) dynamical system can solve computational tasks (Deterministic Finite Automata, Turing Machines, Game of Life, etc.). But LLM systems are quite peculiar: they were not designed to solve a specific task, but rather aim to be  \textit{universal solvers}: given a description of a task, the system should be able to find a solution. Moreover, rather than brute-forcing a solution, we expect it to find the fastest way to solve the problem, as well as accessing past information (e.g., the  ``rational root theorem'' in the example above) to significantly speed up the solution.

\newcommand{\terminates}{\Rightarrow}
\newcommand{\E}{\mathbb{E}}
\renewcommand{\P}{\mathbb{P}}

\subsection{Notation}

In this section we introduce the notation used throughout the paper.\footnote{For those with background in dynamical system theory, in \Cref{app:notation} we map the notation to one more reminiscent of continuous dynamical systems.}
Let $s \in \mathcal{S}$ be a state in a potentially infinite state space $\mathcal{S}$, and let $t \in \mathbb{N}$ be a time index. A sequence of states $h=(s_1, \ldots, s_n)$ is called a trajectory or path. Its length is the time $T(h) = n$. A stochastic dynamical system is defined by the transition probability $\nu(s_{t+1}| s_t)$. We say $h$ is a trajectory between two states $u, v \in \mathcal{S}$ if $s_1 = u$ and $s_n = v$. The probability $\nu(h) = \prod_{t=1}^{n-1} \nu(s_{t+1}|s_t)$ of a trajectory is the product of the transition probabilities along the path.

The system should be able to read inputs and output answers. Let $\Sigma$ be the input/output alphabet.
We assume that the system has a set $\mathcal{F} \subset \mathcal{S}$ of \textit{terminating} states, and a function $\text{dec}: \mathcal{F} \to \Sigma^*$ that, given a terminal state, generates an output. We also assume that there is a function $\text{enc}: \Sigma^* \to \mathcal{S}$ that encodes the input into a state of the dynamical system, where $\Sigma^*$ is the set of all possible finite trajectories. We say that a trajectory $h=(h_1, \ldots, h_t)$ terminates with output $a$ --- which we write $h \terminates a$ ---
if $h_t \in \mathcal{F}$ and $\text{dec}(h_t)=a$.

Let $x \in \Sigma^*$ be an input. For simplicity, we assume that all trajectories starting from $\text{enc}(x)$ and ending in a terminating state terminate with the same output, or with a special \texttt{<error>} token. This allows us to write $\nu(x \terminates a)$ meaning the dynamical system $\nu$ starting from $\text{enc}(x)$ terminates with $a$. While this is generally restrictive, we mainly study settings where the answer is verifiable, in which case we can trivially return error if the output is not correct.\footnote{
  A more standard and less restrictive definition is to ask that the answer is correct at better than chance level, that is we would say $\nu(x \terminates a)$ if $\mathbb{P}[h \terminates a | h_1 = \text{enc}(x)] > 2/3$.
}

\begin{example}{}{}
Two key systems we are interested in are:

\textbf{LLMs.} The state is the set of activations of the LLM after reading some tokens. This is the set of activations (known as the key-value (KV) cache) for an autoregressive Transformer, or more generally the hidden state for a State Space Model. The transition function $\nu(s_{t+1}|s_t)$ generates the next token given the state $s_t$ and uses it as input to generate state $s_{t+1}$.
The final states are the states at which the LLM outputs an \texttt{<end\_of\_thought>} token. The decoding function consists of letting the network generate the answer after \texttt{<end\_of\_thought>}.
The encoding function simply lets the LLM read the input tokens to update its initial state. 

\textbf{Turing Machines.} The state of a Turing machine is the content of its tape at a given time, plus its internal state. The transition function $\nu(s_{t+1}|s_t)$ updates the tape and its internal state as usual, either deterministically in a standard Turing machine, or randomly in a probabilistic machine.
\end{example}

We also make use of several notions from algorithmic information theory \cite{li2008introduction}. Let $x$ be a string, its \textit{Kolmogorov Complexity} $K(x)$ is defined as the length (in bits) of the shortest program that terminates outputting $x$. Given two strings $x$ and $y$, their \textit{algorithmic mutual information} is $I(x : y) := K(x) + K(y) - K(x, y) = K(x) - K(x|y)$ (up to logarithmic additive terms). This can be interpreted as how much more $x$ can be compressed if we have already observed $y$. Recall that, by Shannon's Coding Theorem, given any probability distribution $\nu(x)$ over binary strings, there is a corresponding encoding algorithm that encodes a string $x$ in $\ell_\nu(x) := - \log_2 \nu(x)$ bits.

\subsection{Proper Time}

As we have anticipated, transductive learning is about solving generic tasks quickly. But how do we define the \emph{time} that a \emph{stochastic} system needs to find a solution to the task? The question is subtle, since if we look at the length of a particular sampled trajectory, randomness can make an algorithm look arbitrarily faster or slower, without changing what it effectively computes.

Let us first consider a motivating example. Let $f(x)$ be a function that is easy to evaluate, but can be inverted only through brute-force search (i.e., a `one-way function'). Given $y$, the task is to find a binary string $x$ of length $|x| = n$ such that $y=f(x)$. A \emph{deterministic} Turing machine must try all $2^n$ candidates for $x$, so the total expected time is $T=2^{n-1}$. On the other hand, a \emph{stochastic} machine can \emph{guess} the first $k$ of $x$, and brute-force the remaining $n-k$, so every terminating trajectory has length $T=2^{n-k}$, but occurs with probability only $\nu(h) = 2^{-k}$. 

\begin{figure}[h]
\centering
\includegraphics[width=0.7\linewidth]{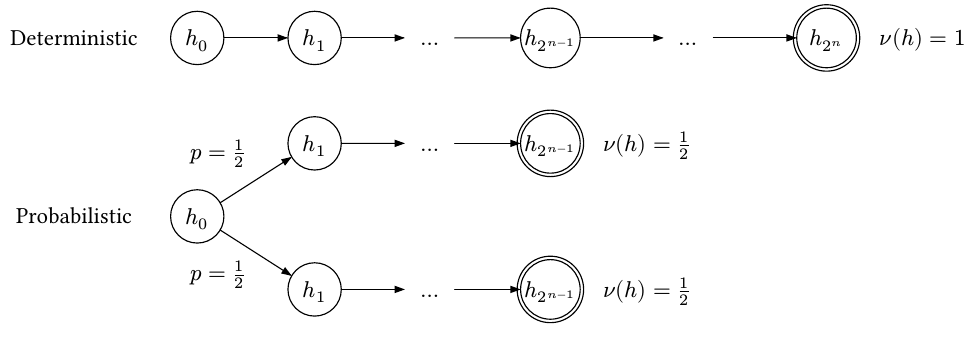}
\end{figure}

The probabilistic machine can then be arbitrarily faster than the deterministic machine as measured by the trajectory length $T(h)$ even if, effectively, both are doing the same brute force search. If however we consider the ratio $\tau(h) = T(h)/\nu(h) = 2^n$, we see that it remains constant, no matter how we branch the computation path: randomness shortens paths but also makes them rarer. This invariance suggests the following definition of “proper” time for single-trajectory targets.

\begin{definition}{Proper Time}{proper-time}
  Let $\nu(h_{t+1}|h_t)$ be a stochastic dynamical system. Define the \emph{proper time} to reach $v$ from $u$ as:
  $$
    \tau_\nu (u \to v) = \min_{h_{u \to v}} \frac{T(h)}{\nu(h|u)}
  $$
  where the minimum is over all trajectories $h$ from $u$ to $v$, or $\tau_\nu = \infty$ if no trajectories exist. For an input–output specification, the proper time to terminate from input $x$ with output $a$ is:
    $$
      \tau_\nu (x \terminates a) = \min_{h \terminates a \text{ and } h_1 = \text{enc}(x)} \frac{T(h)}{\nu(h|\text{enc}(x))},
    $$
    where the minimum is over trajectories starting from $\text{enc}(x)$ and terminating with output $a$.
\end{definition}

This definition is closely related to Levin Complexity \cite{vitanyi_algorithmic} and its extension to tree search \cite{orseau2018single}.
If the system is deterministic, then $\nu(h|u) = 1$ for any path, and $\tau_\nu$ reduces to standard running time. Conversely, we now show that $\tau_\nu$ indeed captures the actual computational effort required by a stochastic dynamical system when simulated deterministically.

\begin{theorem}{Dynamical Systems $\Rightarrow$ Turing Machines with same $\tau$}{levin}
Let $\nu(s_{t+1}|s_t)$ be a dynamical system. There is a deterministic Turing machine $M_\nu$, with access to an oracle to compute $\nu(\cdot | \cdot)$, such that:
$$
  T_{M_\nu} (x \terminates a) \leq 2 \, \tau_\nu (x \terminates a) 
$$
\end{theorem}

The theorem follows directly by taking $M_\nu$ to be the Turing Machine that implements the algorithm in the following key lemma:

\begin{lemma}{Levin Tree Search \cite{orseau2018single}}{proper-time-lemma}
Let $u, v$ be two states. There is a deterministic algorithm $A$ that discovers a path between them (if it exists) while visiting at most $T$ nodes where
$$
  T = 2 \, \tau_\nu (u \to v)
$$
\end{lemma}

\begin{proof}[Proof sketch]
We sketch the argument underlying Levin-Tree Search \cite{orseau2018single}.
For a partial trajectory (search-tree node) $h=(u=s_0,s_1,\dots,s_t)$ define its Levin cost
\[
\tau_\nu(h)\;:=\;\frac{\ell(h)}{\nu(h\mid u)}\,,\qquad 
\nu(h\mid u)=\prod_{i=0}^{t-1}\nu(s_{i+1}\mid s_i)\,,\quad \ell(h)=t .
\]
The algorithm maintains a frontier of unexpanded prefixes and repeatedly expands the prefix having minimum $\tau_\nu$.  Along any root-to-leaf branch, $\tau_\nu$ is non-decreasing (depth increases while probability only decreases), which implies a best-first property: all prefixes with $\tau_\nu \le c$ are expanded before any prefix with $\tau_\nu>c$, where $c:=\min_{h:u\to v}\tau_\nu(h)$ is the optimal goal cost. Let $T_c$ be the (finite) search tree consisting of all expanded nodes at the moment the first goal node is expanded.  Every leaf $h$ of $T_c$ satisfies $\tau_\nu(h)\le c$, hence $\ell(h)\le c\,\nu(h\mid u)$.  Moreover, the number of expanded nodes is at most the sum of the depths of leaves,
$|N(T_c)|\le \sum_{h\in L(T_c)}\ell(h)$, since each leaf contributes at most one count to each of its ancestors.  Therefore
\[
|N(T_c)|\;\le\;\sum_{h\in L(T_c)}\ell(h)\;\le\;c\sum_{h\in L(T_c)}\nu(h\mid u)\;\le\;c,
\]
using that leaf probabilities in a prefix tree sum to at most $1$.  Thus LevinTS reaches $v$ after at most $c=\min_{h:u\to v}\ell(h)/\nu(h\mid u)$ node expansions; translating expansions to visited/generated nodes yields the stated $O(2c)$ bound.
\end{proof}

Since all computation today is executed on deterministic logic hardware, \Cref{theorem:levin} validates $\tau_\nu$ as the ``proper'' way to measure time for a stochastic dynamical system.\footnote{The name has a loose analogy with relativistic \textit{proper time}: like $\tau = t^2 - x^2$, our $\log \tau_v = T - \log v(h)$ mixes temporal and the probabilistic `distances,' providing a representation-invariant clock.}
It also frames $\tau$ as a fundamental property of the algorithm we are executing, rather than a function of the stochasticity of its implementation.

\begin{remark}{}{}
It is useful to compare proper time with other candidate measures of computational cost. Consider a system that, with high probability $1 - \epsilon$, terminates in a short number $T_0$ of steps but, with small probability $\epsilon$, enters an infinite, non-terminating, chain where every transition has probability $\nu(s_{i+1}|s_i) = 1$. If we sample a random trajectory, the \textit{expected time} $\E[T]$ is infinite: even though the non-terminating trajectory is entered with low probability, its infinite length dominates the expectation.\footnote{The same conclusion holds even when all trajectories are forced to terminate. For example, suppose a system has probability $p_n = \epsilon\, 2^{-n}$ of entering a trajectory of length $T_n = 2^n$. Then $\E[T] = (1-\epsilon)\, T_0 + \epsilon \sum_n 1 = \infty$.}
Likewise, the Levin-style quantity $\E\!\left[\frac{T(h)}{\nu(h \mid u)}\right]$ diverges: along the non-terminating chain, the trajectory probability stays bounded away from zero (since each transition has probability one), yet the time grows without bound, so their ratio diverges.

The proper time $\tau_\nu(x \terminates a) = \frac{T_0}{1-\epsilon}$, by contrast, remains bounded and close to $T_0$. Intuitively, proper time discounts trajectories by their probability of being reached, rather than conditioning on having entered them, and so it is not held hostage by low-probability pathological paths.

This distinction has a practical consequence: to harness the computational power of a stochastic system, it is not enough to sample trajectories from it. Instead, a dovetailing strategy, as in \Cref{lemma:proper-time-lemma}, should be used (in the next section, we also show that a sampling approach with a suitable restart schedule can achieve close to optimal time). More broadly, this highlights that a standalone LLM is not itself an optimal solver, since its time to terminate can be arbitrarily large. Rather, an optimal solver \textit{agent} is constructed from the LLM by wrapping it in an appropriate search procedure.
\end{remark}

In a deterministic system, the time (path length) between states is a distance. Similarly, the following theorem establishes that for a path $x \to y \to z$, the proper time to go from $x \to z$ cannot be greater than the time it takes to first go to $y$ and then to $z$. It will play an important role in multiple proofs.

\begin{lemma}{Proper Time is submultiplicative}{submultiplicative}
Let $x, y, z$ be three states. Then:
$$
  \tau_\nu (x \to z) \leq \tau_\nu (x \to y) \cdot \tau_\nu (y \to z)
$$
\end{lemma}

\begin{proof}
Let $h_{x \to y}$ and $h_{y \to z}$ be paths that realize the minimum in the definition of $\tau$.
We can construct the path $h_{x \to z} = h_{x \to y} \circ h_{y \to z}$ composing the two paths. By \Cref{definition:proper-time} we have:
\begin{align}
  \tau_\nu (x \to z) &\leq \frac{T(h_{x \to z})}{\nu(h_{x \to z})} 
  = \frac{T(h_{x \to y}) + T(h_{y \to z})}{\nu(h_{x \to y}) \nu(h_{y \to z})} \\
  &{\leq} \frac{T(h_{x \to y}) \, T(h_{y \to z})}{\nu(h_{x \to y}) \,\nu(h_{y \to z})} \label{eq:*}\\
  &= \tau_\nu (x \to y) \, \tau_\nu (y \to z)
\end{align}
Where in the second line \eqref{eq:*} we used the fact that $T_1 + T_2 \leq T_1 T_2$ whenever $T_1 \geq 2$ and $T_2 \geq 2$ which is automatically satisfied when the states are distinct. If two or more states are the same, the property can be easily checked by hand.
\end{proof}

This also implies that $\log \tau(x \to z) \leq \log \tau(x \to y) + \log \tau(y \to z)$, which makes $\log \tau$ an \textit{asymmetric distance} between states. Note that $\tau$ is sub-multiplicative, while deterministic time is sub-\emph{additive}. This is because in a stochastic system, time may be dominated by the time to find a suitable combination of paths, and probability of the composition of two paths is a product.

These two lemmas suffice to prove the key theorems in the rest of this work, including providing a straightforward construction for a generalization of Solomonoff-Levin Universal Search Algorithm \cite{levin1973universal,solomonoff1984optimum}.

\subsection{Multiple successful paths}

The quantity in \Cref{definition:proper-time} measures the cost to \textit{uncover a particular trajectory}. Many tasks, however, accept \textit{any} trajectory leading to one of the final states $s \in \mathcal{F}$. In that setting, multiple distinct paths can succeed, and the right notion aggregates their probabilities. Let
$$
F_\nu(t;\mathcal{F})
:= 
\Pr_\nu\bigl(\text{reach }\mathcal{F}\text{ within } t \text{ steps}\bigr)
$$
be the \textit{success-by-time} curve. If we run independent trials of length $t$ (restarting after $t$ steps), we need in expectation $1/F_\nu(t;\mathcal{F})$ trials for one success, for total expected work $t/F_\nu(t;\mathcal{F})$. Optimizing over the cutoff gives a canonical baseline:
$$
\tau_\nu^*(\mathcal{F})
:=
\inf_{t\ge 1}\;
\frac{t}{F_\nu(t; \mathcal{F})}.
$$
This general notion (i) strictly improves over any single-path bound when many solutions exist, and (ii) collapses to the proper time when there is effectively one successful path. In principle, to simulate the stochastic system in total time $\tau_\nu^*(\mathcal{F})$ we would need the unknown optimal cutoff $t$. However, universal Luby-type restart schedules \cite{luby1993optimal,orseau2018single} achieve expected work within a logarithmic factor of the optimum fixed-cutoff policy:
\[
\text{Expected work }=\; O \bigl(\tau_\nu^*(\mathcal{F})\,\log \tau_\nu^*(\mathcal{F})\bigr).
\]
Thus $\tau_\nu^*(\mathcal{F})$ characterizes intrinsic difficulty `up to logs.'
For clarity of exposition, in the rest of the paper we focus on $\tau_\nu(h)$, but all results extend naturally to $\tau_\nu^*(\mathcal{F})$.

\subsection{Universal Dynamical Systems}

A key part of defining a computation system is the existence of universal systems. Generally, we want a system in the class to be able to simulate any other system in the same class. However, since time is a key quantity for us, we need to ensure the time to simulate is similar to the original time.

\begin{definition}{Linear-Time Universal Dynamical System}{linear-time-universal}
Let $\nu$ be a dynamical system, and let $[\nu]$ be its encoding. We say that dynamical system $U$ is \textit{linear-time universal} if for any $\nu$ we have:
\[
  U(x[\nu] \terminates a ) \Leftrightarrow  \nu(x \terminates a) \text{ and } \tau_U (x \terminates a) \leq C_\nu \tau_\nu (x \terminates a)
\]
for some constant $C_\nu$ which depends on $nu$ but not on the input $x$.
\end{definition}

Since for any dynamical system $\nu$ there is a Turing machine $M_\nu$ which simulates $\nu$ in the same proper time, to satisfy the definition it is enough to verify that for any Turing machine $M$ we have:
\begin{equation}
\label{eq:turing-universal}
   U(x[M_\nu] \terminates a ) \Leftrightarrow  M(x) = a \text{ and } \tau_U (x \terminates a) \leq C_M T_M (x \terminates a)
\end{equation}
which is generally easier to verify. This observation trivially gives us:
\begin{corollary}{}{}
  There exists a Linear-Time Universal Dynamical System.  
\end{corollary}
\begin{proof}
    Any linear-time universal Turing machine is a dynamical system and by definition can emulate other Turing machines in linear time.
\end{proof}

\section{Universal Solvers}

We are now finally ready to introduce Universal Solvers, which are our main focus. A universal solver is a dynamical system that can efficiently find a solution to an arbitrary problem, if one exists. We formalize it as follows.

Let $f(x, y): X \times Y \to \{0,1\}$ be a computable function. We say that $y$ is a witness of $x$ if $f(x, y)=1$.
A universal search program is any program $S$ that, provided with an oracle for $f$ and an input $x$, terminates with outputting $y$ such that $f(x,y)=1$ (we generalize this to continuous rewards in \Cref{sec:continuous-reward}): 
$$
  S(x \terminates y) \iff f(x, y)=1
$$
If $y$ does not exist, the program is allowed to terminate with an error or not terminate at all.
Generally, together with the input $x$ we may also pass a description of the objective $f(x,y)$ so the search program is not blind. To keep the notation uncluttered, we do not denote this additional input.

It is always possible to find a witness to any problem by just enumerating all possible $y$ in a dovetail fashion and checking for $f(x,y)=1$ using the oracle for $f$. However, we are interested in search programs that are as efficient as possible.

\begin{definition}{Universal Solver}{universal-solver}
  A dynamical system $U$ is a \emph{universal solver} system if, for any objective $f(x,y)$ and any other system $A$ that solves the problem --- i.e.,  for all $x$, $A(x \terminates y)$ with $f(x,y)=1$ --- we have:
  $$
    \tau_U (x \terminates y) \leq C_A \tau_A (x \terminates y).
  $$
\end{definition}

That is, for any task, a universal solver is at most a constant factor slower than the best possible system $A$ that solves that particular task. The existence of a universal solver is non-trivial. Levin introduced the notion of universal search, as well as a sketch of the existence of such a system in the same paper that introduced the notion of NP-Complete problems \cite{levin1973universal}. Solomonoff later realized its importance for machine learning, and provided a detailed proof \cite{solomonoff1984optimum}. With the formalism we introduced, the existence proof is straightforward and can be generalized to any stochastic system, with Turing Machines as a special case.

\begin{theorem}{Existence of Dynamical System Universal Solvers}{universal_search}
  Let $m$ be any distribution encoding programs from which we can sample. Then, there is a dynamical system $U_m$ such that for any solver $A$:
  \[
    \tau_{U_m} (x \terminates y ) \leq C'_A 2^{-\log m(A)} \tau_A (x \terminates y).
  \]
  In particular, $U_m$ is a universal solver with constant $C_A = C'_A 2^{-\log m(A)}$.
\end{theorem}

\begin{proof}
  Let $U$ be any linear-time universal system as in \Cref{definition:linear-time-universal}. Construct a composite dynamical system $U_m$ as follows. First, given $x$ use $m$ to sample a program encoding, call it $[A]$, and append it to the input to get $x[A]$. Then run the universal system $U$ to execute $[A]$ on $x$. 

  \begin{figure}[h]
  \centering
  \includegraphics[width=0.6\textwidth]{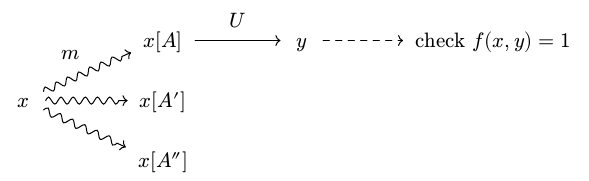}
  \end{figure}

  Let $A$ be any algorithm that solves the search problem. Then $U(x[A] \terminates y)$, giving us a path $x \to x[A] \to y$ from the input $x$ to the solution $y$. Applying the submultiplicativity of $\tau$ to this path from \Cref{lemma:submultiplicative}, and using the definition of universal dynamical system (\Cref{definition:linear-time-universal}) we have:
  \begin{align}
    \tau_{U_m} (x \terminates y) &\leq \tau_m (x \to x[A]) \tau_U (x[A] \terminates y) \\
    &\leq C'_A \tau_m (x \to x[A]) \tau_A (x \terminates y)
  \end{align}
  Note that by definition of $U_m$ we have $\tau_{U_m} (x \to x[A]) = T/m(A) = 1/m(A)$ (where we assume the entire program $[A]$ is sampled in one step). Replacing this identity in the above we get:
  $$
    \tau_{U_m} (x \terminates y) \leq C'_A 2^{-\log m[A]} \tau_A (x \terminates y)
  $$
  Which gives the desired time bound.
\end{proof}

The construction above instantiates a particular universal solver which first `guesses' a program that may solve the task and then execute it. Of course, in general, universal solvers need not be a one-shot guess: human problem-solvers will not blindly guess an algorithm and execute it, but will rather interleave partial computations,  observations,  backtracking and shortcuts. Our general stochastic dynamical-system view already subsumes such interactive behavior. Nonetheless, the search algorithm presented is universal (as in, no other can be significantly faster) and already allows us to make some general observations:

\textbf{Speed of a universal solver.} For any solver $A$ that succeeds on $x$, the universal solver above solves the problem in time:
$$
  \tau_{U_m}(x \terminates y)\;\le\; C_A \, 2^{-\log m(A\mid x)} \, \tau_A(x \terminates y),
$$
that is, the slowdown with respect to an arbitraty solver is the simulation constant $C_A$ times the inverse prior weight of the right program. 
This bound highlights two levers for learning. The term $2^{-\log m(A\mid x)}=2^{\ell_m(A\mid x)}$ depends on the code length $\ell_m$: if tasks reuse a small set of subroutines, reshaping $m$ to give these short codes yields exponential gains (we return to this in \Cref{sec:scaling-laws}). The factor $C_A$ reflects how many steps our base dynamics spend simulating a single step of $A$; when particular transition patterns recur, we can `macro-step’ them—effectively adding shortcuts in the dynamics—to shrink $C_A$. These ideas extend beyond the guess–execute prototype to any universal search system and we will study them in details in later sections.

\textbf{Universal Search and Universal Computation.} From the proof of \Cref{theorem:universal_search}, we see that starting from a universal system we can easily construct a time-optimal universal search system. The following straightforward theorem shows that the converse also holds. For a system to be a universal solver it needs to be a universal computation system.

\begin{theorem}{Universal Search $\Rightarrow$ Linear-Time Universal System}{universal-computation}
  Let $U$ be an optimal universal search program. Then it is also a universal dynamical system.
\end{theorem}

\begin{proof}
  Let $M$ be a Turing machine. Construct the function $f(x,y) = 1$ if $M(x)=y$ and 0 otherwise.
  By definition of optimal universal search we have $U(x[f]) \terminates y \iff f(x,y)=1$ which implies $M(x) = y$ by construction. Moreover
  \begin{align}
  \tau_U (x) &\leq C_M \tau_M (x \to x[M]) T_M(x) \\
  &= C'_M T_M (x)
  \end{align}
  therefore it is linear-time universal.
\end{proof}

The claim is straightforward but it has an important implication: \emph{if we train a model to solve a sufficiently general set of tasks, the model will necessarily learn to simulate a Universal Turing Machine}. Whether this has already happened in the current generation of LLMs pretrained on language has been hotly debated. Depending on the exact setting and access to tools or external memory, one can argue one way \cite{schuurmans2023memory, papageorgioucan,schuurmans2024autoregressive} or the opposite \cite{borji2025illusion,merrill2023expressive}.

\textbf{Optimality through meta-learning.} An important property of a universal search system is that it is necessarily optimal, meaning that no other universal search system can be arbitrarily faster:

\begin{lemma}{}{}
  Let $U$ and $U'$ be two universal search systems. Then, for any task $f$
  we have:
  $$
    \tau_U (x \terminates y) \leq C_{U'} \tau_{U'} (x \terminates y).
  $$
  where the constant $C_{U'}$ does not depend on the task $f$ or the input $x$.
\end{lemma}

\begin{proof}
    Since by universality $U'$ finds the solution to the task $f$, we can take $A = U'$ in the definition of universal search, giving us
  $$
    \tau_U (x \terminates y ) \leq C_{U'} \tau_{U'} (x \terminates y)
  $$
  Hence, $U$ is at most $C_{U'}$ times slower than $U'$, where $C_{U'}$ does not depend on the task $f$.
\end{proof}

The proof is a trivial manipulation of the definitions, but it underlies a key concept,  which in modern terminology would be called \emph{meta-learning}. Let's use the particular universal system in \Cref{theorem:universal_search} to make the point explicit. 
For it, the time required to solve a task depends on $-\log m(A)$, the encoding length of its optimal solution. It is \textit{a priori} possible that a system $U'$ may achieve a better time on some tasks by learning a better encoding $m'$ specific for them. However, $U$ can just search (meta-learn) the solver $U'$, and use it to solve the task leading to a slow down of at most $2^{-\log m(U')}$. In practice, the constant $2^{-\log m(U')}$ is too large, and we need to amortize it through learning, which is our focus for most of this work.

\subsection{Universal Solvers and Sampling}

By \Cref{theorem:levin} we can convert a stochastic system $\nu$ to a deterministic program that finds a solution in time $T = \tau_\nu (x \mapsto y)$. However, this program is \emph{not} obtained by naively sampling a random trajectory up to completion, as one may be tempted to do. In fact, doing that would have an infinite expected runtime:

\begin{example}{}{}
Let $\nu$ be any computable prior that gives non-zero mass to all programs, (e.g., the Universal Prior), then $\mathbb{E}_{A \sim \nu}[T_A] = \sum \nu(A) T_A = \infty$ even assuming we have an oracle preventing us from running algorithms that do not terminate. To see it, consider the program $A$ that computes $\nu(A)$ and runs for $\nu(A)^{-1}$ steps before terminating. Then $\nu(A) T(A) = 1$ and there are infinite such programs in the expectation.
\end{example}

This highlights an important principle: if we have a way to guess a possible solution, in general it is not time-optimal to keep generating guesses and testing them. For an LLM, this means that sampling CoT traces until one succeeds is not a good idea. Rather, we need to keep open multiple possibilities and smartly allocate time budget between all of them. To add some color, imagine trying to prove a theorem. You will likely start with the most likely guess, but if it starts to take too long with no solution in sight, you will try spending some time on another approach and perhaps come back to the original approach later.

The construction in \Cref{theorem:levin} which achieves $\tau$ on a deterministic system can be made into a stochastic algorithm. The algorithm above hinges on keeping multiple hypotheses at the same time and continuing to explore them with increasingly more budget. What prevents us from having a system that achieves the same expected time by sampling individual trajectories?

We have seen before that such a system cannot sample programs directly from $\nu(A)$ as the expected time could easily be infinite. A good guess is that we need to sample from the distribution\footnote{This distribution is closely related to Schmidhuber's Speed Prior \cite{schmidhuber2002speed} and Filan et al.'s $S_{\text{Kt}}$ prior \cite{filan2016loss}.}
\begin{equation}
\nu_t (A) = \frac{1}{Z} \frac{\nu(A)}{T(A)},
\end{equation}
where $Z$ is the normalization constant, which prioritizes programs that have a short running time. This is indeed the case:

\begin{theorem}{Time-Weighted Sampling}{}
Let $\nu$ be a universal search system. If we sample trajectories from:
\begin{equation}
h \sim \nu_t (h|u) \propto \frac{\nu(h|u)}{T(h)}
\end{equation}
and run them to completion, the total amount of operations we need to perform before finding a solution is:
\begin{equation}
\mathbb{E}[T_{\text{total}}] = \tau_\nu (x \mapsto y).
\end{equation}
\end{theorem}

\begin{proof}
Let $h_*$ be a trajectory solving the task, and let $n_*$ denote the number of iterations before $h^*$ is sampled. In expectation, we have $\mathbb{E}[n_*] = \frac{1}{\nu_t(A)} = Z \frac{T(h)}{\nu(h)}$. We now need to compute how much time is spent validating each of the $n_*$ samples. The expected time that we need to spend validating a single sample from $\nu_t$ is:
\begin{equation}
\mathbb{E}_{A \sim \nu_t}[T_A] = \frac{1}{Z} \sum_i \frac{\nu(A_i)}{T_{A_i}} T_{A_i} = \frac{1}{Z}
\end{equation}
so the total time we need to spend validating the $n_*$ is:
\begin{equation}
T_{\text{total}} = \mathbb{E}[T_1 + T_2 + \ldots + T_{n^*}] = \E[n_*] \mathbb{E}[T_i] = Z \frac{T(A)}{\nu(A)} \frac{1}{Z} = \frac{T(A)}{\nu(A)}
\end{equation}
which gives the desired result.
\end{proof}

Hence, a universal search algorithm that only wants to consider one guess at the time has to learn how to sample from $\nu_t (h)$, which means that in addition to estimating the probability $\nu(h)$ that a solution is correct, it should also be able to predict the time $T(h)$ that it will take to run it.

The distribution $\nu_t$ is actually computable (the time $T(h)$ may be undecidable, but to upper bound $\nu_t (h)$ within $\epsilon$ we just need to show that $\nu(h)/T(h) < \epsilon$ and hence run for $T = \nu(h)/\epsilon$ steps). However,  in \cite{filan2016loss} it is shown that $\nu_t (h)$ takes double exponential time in $1/\epsilon$ to approximate, and doing so essentially requires running multiple programs, which we want to avoid in the first place.

Hence, the only option left if we want to avoid searching over trajectories is to \textit{train} a system to approximate both the likelihood of solution and the cost of time. While this will not be our focus, the following importance weighted training scheme gives a way to train:

\begin{theorem}{Importance-Weighted Training}{}
Let $\nu$ be a dynamical system. Let $\mathbf{h} = (h^1, \ldots, h^n)$ be a batch of trajectories sampled from $\nu$. Then the distribution $\mu^*$ minimizing:
\begin{equation}
\mu^* = \arg \min_\mu \mathbb{E}_{\mathbf{h}} \left[\sum_i w_i \mu (h^i)\right] \quad \text{with} \quad w_i = \frac{T(h^i)}{\sum_j T(h^j)}
\end{equation}
is exactly $\mu^* = \nu_t$.
\end{theorem}

\section{Scaling Laws for Speed}\label{sec:scaling-laws}

By definition of universal solver, given a function $f(x,y)$ and an input $x$, there is a trajectory $h$ finding a witness $x$ if such a witness exists. Let $h$ be the shortest such trajectory,  i.e., the one with minimal $T(h)$. The total time required by the universal search system to find it is:
\begin{equation}
\tau_U = 2^{-\log \nu(h)} T(h) = 2^{\ell_\nu (h)} T(h)
\end{equation}
where we defined $\ell_\nu (h) = -\log \nu(h)$ to be the compression length of the trajectory using $\nu$. How do we reduce the search time $\tau_U$? We could reduce the thinking time $T(h)$ by learning to skip some steps to get directly to the solution. But the largest improvement will come from reducing the exponential factor $2^{\ell_\nu (h)}$. This is the time needed to \emph{guess} the correct solution to the problem. Thanks to Shannon's Coding Theorem, we can improve the probability of guessing the solution, thus speeding up the search, by instead finding a way to \emph{reduce the compression length} of $h$. We can do this by learning from a dataset $D$.

For example, suppose we have a list of programs that have worked well in the past. If we notice that some pieces of code tend to appear frequently (say, the code to compute an FFT), we could change the encoding to replace those pieces of code with a unique name. This reduces the length of those programs making them more likely to be sampled. Not only that, but any program reusing those components is more likely to be guessed in the future.

Another example to add color: suppose that while proving theorems we often use the same sequence of steps. We probably will want to turn it into a named theorem --- e.g.,  ``Cauchy–Schwarz inequality'' --- which will also make us more likely to try to use it in future problems. In this spirit, let's crystallize this in the following:

\begin{theorem}{Better compression $\Leftrightarrow$ Faster search}{}
For a universal search system with model $\nu$, improving the compression of a trajectory $h$ by $\Delta$ bits accelerates its discovery by a factor of $2^\Delta$.
\end{theorem}

Let's now formalize what learning from data means. Given some data $D$, we denote by:
\begin{equation}
\ell_\nu (h|D) = -\log \nu(h|D)
\end{equation}
the negative log-likelihood given by the model to a trajectory after observing the dataset $D$. One possibility is that we \emph{train} on the data $D$. In this case, we assume that $\nu_\theta$ is a parametrized family of distributions. Let $\theta_D$ be the parameters obtained after training on $D$. Then we define $\nu(h|D) := \nu_{\theta_D} (h)$ as the likelihood given to $h$ by the trained model. Alternatively, we can do \emph{in-context learning}  (ICL) or \emph{prompting} where we feed the data $D$ to the model to obtain a state $s_D$, and then we set $\nu(h|D) := \nu(h|s_D)$ the likelihood of the trajectory after having seen the data. It could also be that the model $\nu(h)$ has a way to retrieve information from $D$, a process known as \textit{retrieval-augmented generation} (RAG). And any mix of these methods may be used (some data is used to train, other to prompt, other is used for retrieval). While different in implementation, from a theoretical perspective there is no fundamental difference between these ways of using past data, and which is why we can write generically $\nu(h|D)$.

In our setting, the benefit of learning is not measured by better accuracy --- since we have a verifier, sooner or later, we \emph{will} find a correct solution --- but rather by the reduction in search time. The speed-up factor achieved after training on the data is given by the ratio:
\begin{equation}
\label{eq:proper-time-ratio}
\frac{\tau_\nu (h)}{\tau_{\nu_D} (h)} = 2^{\ell_\nu (h) - \ell_\nu (h|D)} = 2^{I_\nu (h : D)}
\end{equation}
where we defined the $\nu$-algorithmic mutual information:
\begin{equation}
I_\nu (h : D) = \ell_\nu (h) - \ell_\nu (h|D)
\end{equation}
So, the speed up of universal search is given by the algorithmic mutual information between inference time trajectories and past data. We highlight this in the following:

\begin{theorem}{Information is speed}{information-speed}
The log-speedup of a search algorithm after seeing data $D$ is:
\begin{equation}
\log \frac{\tau_\nu (h)}{\tau_\nu (h|D)} = I_\nu (h : D).
\end{equation}
\end{theorem}
\begin{remark}{}{}
    \Cref{theorem:information-speed} follows easily from the straightforward algebraic manipulation in \cref{eq:proper-time-ratio}. This is possible thanks to the work we already did in defining and validating the notion \textit{proper time}, which connects the likelihood of trajectories in a dynamical system (which relates to information) with the time necessary for the system to find a solution.
\end{remark}

We are interested in $\nu$ that are very good compressors (since we want to minimize $\ell_\nu (h)$). Asymptotically, the best compressor is the universal prior $m(h) \propto 2^{-K(h)}$ for which $I_m (X : Y)$ becomes \textit{the} algorithmic mutual information \cite{li2008introduction}:
\begin{equation}
I (h : D) = K(h) - K(h|D) = K(h) + K(D) - K(D h).
\end{equation}
While we are interested in $I_\nu (h : D)$, we can use $I (h: D)$ as a proxy of what is the best we could achieve asymptotically. The advantage is that $I(h : D)$ has a number of theoretical properties that make it easier to work with.

The key question now is: what is the maximum possible log-speedup $I(h: D)$ we can get from learning? As it turns out, the answer is not straightforward and depends on some key assumptions about how real data works. Let's get there step-by-step.

The trajectory $h$ is the trajectory of an optimal solution to a task (e.g.,  the optimal CoT to get to a solution, or the shortest execution trace of a program that solves the problem). Meanwhile $D=\{h_1, \ldots, h_n\}$ is presumably created by collecting examples of trajectories that optimally solved tasks in the past.

A first guess (often done in the Minimum Description Length literature) is that solutions to real world problem tend to have low complexity. It therefore may make sense to hypothesize that $h_i \sim m(h) = 2^{-K(h)}$ is sampled from the universal prior itself, which favors low-complexity solutions. What would be $I(h : D)$ in this setting? Disappointingly, we can show that:
\begin{equation}
\mathbb{P}[I(h : D) > k] \leq n c 2^{-k}
\end{equation}
so the probability that past data $D$  share substantial information with the solution to the present task $h$, and therefore can lead to substantial speedup through learning, is vanishingly small. This does not bode well for the possibility to learn a fast universal solver. 

To see what happened, it is useful to abstract a bit. Suppose we have a mechanism $q(h)$ generating trajectories. Let $D=\{h_1, \ldots, h_n\} \sim q$ be data seen in the past (our training set), and let $h_{\text{new}} \sim q$ be a new data we are trying to find at inference time. This forms a graphical model:

\begin{center}
\vspace{0.5em}
\resizebox{.3\textwidth}{!}{
\begin{tikzpicture}
  \node[latent]                     (q)   {$q$};
  \node[latent, below left=1.0cm and 2.4cm of q] (h1)  {$h_1$};
  \node[ right=0.2cm of h1]          (hdots)  {$\ldots$};
  \node[latent, right=0.2cm of hdots]          (hn)  {$h_n$};
  \node[latent, below right=1.0cm and 1.2cm of q] (hnew){$h_{\text{new}}$};

  \plate {plateD} {(h1)(hdots)(hn)} {$D=\{h_1,\ldots,h_n\}$};

  \edge {q} {plateD};
  \edge {q} {hnew};
\end{tikzpicture}
}
\vspace{0.5em}
\end{center}

where $q$ acts as a separator between past and future data. By the Data Processing Inequality \cite{cover1991information} this implies
\begin{equation}
I(D: h_{\text{new}}) \leq I(D: q) \leq K(q).
\end{equation}
That is, since $h_{\text{new}}$ is sampled i.i.d. from $q$, the only information that the past data $D$ can provide about $h_{\text{new}}$ is information about $q$ itself, and this cannot be larger than its description length $K(q)$.

\begin{theorem}{Maximum speedup is bound by world complexity}{max-speed-up}
The maximum speed-up we can obtain using data generated by a process $q$ is:
\begin{equation}
\log \frac{\tau_\nu (h)}{\tau_\nu (h|D)} = K(q)
\end{equation}
\end{theorem}

Since the universal prior $m(h)$ has low Kolmogorov complexity $K(m) = O(1)$,\footnote{This may be slightly confusing, $m$ can generate programs of arbitrary complexity, but \emph{its own complexity} is low. In fact, we just need a few lines to define it.}
there is nothing we can learn from it. More generally, whenever the data is generated by a low-complexity distribution, no matter how much data we observe, we will never be able to obtain more than a constant time speed up.

This gets to a key question about what is the scaling law of information for real world data. To study it further, it is useful to reframe the question a bit. We have been thinking of $q$ as a mechanism that generates i.i.d. sample of trajectories. This may be restrictive. More generally, let $q$ be a dynamical process generating a sequence $x_1, x_2, \ldots$ of tokens. Let $X_n = x_{1:n}$ and $Y_m = x_{n:n+m}$ be an initial sequence of length $n$ and its continuation of length $m$. We can think of $X_n$ as our training set (past data) and $Y_m$ as our test set (future data we are trying to predict). It may be useful to think of $X_n$ and $Y_m$ to be natural language text, or code.

We want to know how $I(X_n; Y_m)$ scales when $n, m \to \infty$. Let's suppose $q$ is a finite-dimensional Markov process with a discrete $D$-dimensional hidden state $s \in S^D$ over some alphabet $S$. What information can $X_n$ provide about $Y_m$? By the Markov hypothesis, the only information that $X_n$ can provide to help predicting $Y_m$ are the parameters $\theta$ of the underlying process, and the final state $s_n$, so we have:
\begin{equation}\label{markov-information-bound}
I(X_n : Y_m) \leq c |\theta| + D \log |S|
\end{equation}
where $|\theta|$ is the number of parameters of the process, and $c$ is how many bits we need to encode the parameters. Again, we find that for a very large class of processes $I(X_n : Y_m)$ is bounded by a constant, and asymptotically there is nothing to learn as long as (i) the parameters of the process are finite-dimensional and (ii) the size of the state is bounded (or, equivalently, the process has finite or fading memory).

\subsection{Hilberg's law for scaling}

Is this what happens on real data? A particularly well studied case is when the process $q$ generating the data is a human writing natural language text. In the special case that $n=m$, the \textbf{Hilberg's law} \cite{hilberg1990bekannte,debowski2011vocabulary,debowski2012mixing,debowski2020information,dkebowski2023multiperiodic,dkebowski2025zipf}, posits that:
\begin{equation}
I(X_n : Y_n) \propto n^\beta
\end{equation}
for some $0 < \beta < 1$. This is in sharp contrast with the results above. If Hilberg's law holds (which, empirically, it does \cite{takahira2016entropy}), then the process generating real data is very unlike any standard dynamical process.\footnote{An unrelated consequence is that a pure LLM implemented by a \emph{state space model} or an attention model with finite context cannot possibly be a perfect model for natural language. Since its state is bounded, it satisfies \cref{markov-information-bound} and cannot asymptotically scale like natural text. However, RAG sidesteps the issue, so an agent with external memory can be a model of language, or a model of the world, in ways in which an ordinary Transformer cannot no matter how many parameters it has and how much data it is trained on.}

Since we care about real data, let's introduce the following generalized Hilberg's law (GHC) scaling to arbitrary $n$ and $m$, and take it as our assumption of how physically-generated data, including human-generated ones, behave.

\begin{definition}{Generalized Hilberg's law}{}
Let $X_n = x_{1:n}$ and $Y_m = x_{n:n+m}$ be data generated by a stochastic process. We say that it has GHC scaling if:
\begin{equation}
I(X_n : Y_m) \propto n^\beta + m^\beta - (m + n)^\beta
\end{equation}
\end{definition}

This expression reduces to the standard Hilberg's law when $n=m$. It is symmetric, and is always positive.\footnote{Define $s=n/(m+n)$ and $t=m/(n+m)$. The function $f(x) = x^\beta$ is convex, so $s^\beta + t^\beta \geq (s + t)^\beta =1$.}
To get an intuition of how a process may satisfy the GHC, in \Cref{sec:santa-fe} we will show one explicitly based on the \emph{Santa Fe process} \cite{dkebowski2023multiperiodic,dkebowski2025zipf}. The key intuition will be that the GHC is satisfied whenever the ``world'' (whatever is generating the data) contains an unbounded amount of unchanging \emph{facts} that are referenced in the data with a heavy tail distribution. For now, let's assume our process satisfies the conjecture and derive the scaling laws for speed up of a universal search agent.

\subsection{Scaling Laws for Time}

Assume the training data $X_n$ and the inference data $Y_m$ are generated by a process satisfying the GHC: 
\begin{equation}
I(X_n; Y_m) =  m^\beta + n^\beta - (n+m)^\beta.
\end{equation}
We are interested in the case where $X_n$ is the training set, so $n \gg m$, in which case we can approximate:
\begin{equation}
I(X_n; Y_m) \approx m^\beta - \beta \frac{m}{n^{1 - \beta}}\ .
\end{equation}
From \Cref{theorem:information-speed}, the log-speed-up we get from training is exactly $I(X_n; Y_m)$ and $m=T(h)$ is the length of the inference-time trajectory. Therefore we conclude:

\begin{theorem}{Time Scaling Law}{time-scaling-law}
The log-speed up we obtain training on a large enough dataset $D$ of $n$ tokens is:
\begin{equation}
\log \frac{\tau_\nu (h)}{\tau_\nu (h|D)} = T(h)^\beta - \beta T(h)/n^{1-\beta}
\end{equation}
\end{theorem}

This tells us a few interesting things. First, the speed up is upper-bounded not by a constant (like we previously obtained for simple models) but by $T(h)^\beta$. That is, the longer the trajectory is, the more it is sped up by learning. This makes intuitive sense: if finding a solution required just a few steps, even without any learning we could have brute-forced it quickly. Complex problems are the ones that benefit the most from learning. We also get $O(T(h)/n^{1-\beta})$ convergence to the optimal speed up, so we want the number of training tokens $n$ to be:
\begin{equation}
n \propto L^{1/(1-\beta)}.
\end{equation}
where $L$ is the maximum length of a trajectory we expect to need to solve a problem. That is, we need more training tokens if we expect to solve challenging problems.

The parameter $0 < \beta < 1$ relates to the \textit{complexity} of the task distribution; in particular, it controls how long-tailed the distribution of `useful facts' is, with $\beta \to 1$ implying that the distribution is very heavy tailed. When $\beta$ is high, we need significantly more training tokens to achieve the optimal rate, since there are many more facts that are commonly used. But we also get a better payback, since the speed up $T(h)^\beta$ is also going to be larger.

\begin{remark}{}{}
For natural language, $\beta \approx 0.8$ \cite{takahira2016entropy}, which gives $n \propto L^{1.25}$. So if we are going to generate trajectories of $10K$ tokens, we need $\approx 100k$ training tokens.
\end{remark}

This ratio of test to train data is realistic when fine-tuning a model for reasoning. But when training from scratch it is a clear underestimate. There are a few factors to consider. First, the initial training data is needed to put the weights in a proper configuration, which depends more on the amount of weights than on the information in the training data. Indeed, it is common to pretrain on lower-information content with size proportional to the number of weights. Second, we are assuming that the mechanism generating the test data is the same as the training data, which is not the case. Facts that are useful at test time may appear very rarely in the training set (e.g., if we ask PhD-level questions to a model trained on generic data). Third, the scaling laws are derived under the assumption that we can identify useful facts and memorize them the first time we see them. But realistically, we need to see a fact multiple times to identify it as useful, which inflates the number of required tokens.

\subsection{Memory-Time Trade-Off}

So far we have assumed we can use all information in $X$, but in practice the available memory $M$ may be a bottleneck.
On a dataset of length $n$ there are $k=n^\beta$ facts to memorize, requiring $M = ck$ bits of memory. Replacing $n$ with $M$ using this relationship in \Cref{theorem:time-scaling-law} we get the scaling if memory (rather than $n$) is the bottleneck.

\begin{corollary}{Time-Memory Scaling Law}{}
Assuming memory is used optimally, the speed-up as a function of the used memory is given by:
\begin{equation}
\log \frac{\tau_\nu (h)}{\tau_\nu (h|D)} = T(h)^\beta - \frac{T(h)}{M^{1/\beta-1}}
\end{equation}
\end{corollary}

However, this assumes that we are somehow able to extract from the training data the most useful facts and store them (and only them) in memory. Since we are using an online learning algorithm, the memory also needs to store information about the facts in the training data that we have not yet deemed useful, since we need to wait to see them again to confirm if they are useful.

\begin{proposition}{Online Memory Overhead}{}
An online agent needs a constant factor
\begin{equation}
M_{\text{online}} = C_\beta M_{\text{offline}}
\end{equation}
of additional memory compared to an offline to achieve the same performance.
\end{proposition}

This reflects a realistic issue: it is easier (i.e.,  faster) to learn from a textbook that gives us directly the useful facts (offline learning) rather than having to `connect the dots' and try to guess the useful facts from online experience.

\subsubsection{Prompting and RAG}

So far we have focused on the speed-up provided by training on a dataset $D$ of past data. What is instead the effect of adding a prompt $p$ to the request? First, note that the key result:
\begin{equation}
\text{speed-up} = 2^{I_\nu (h : p)}
\end{equation}
where the dataset $D$ is replaced by the prompt $p$, remains valid, so the speed up is still determined by the $\nu$-algorithmic mutual information between the prompt and the trajectory.

If the prompt is an \textit{in-context learning} prompt, which provides examples of the task, then the theory is identical to the case of a dataset (effectively the prompt \emph{is} a dataset). However, we expect it to provide much more algorithmic information per-sample than the pre-training dataset $D$, since presumably it will contain only examples directly relevant to the task.

The prompt could also contain information directly relevant to the trajectory, which does not follow a GHC scaling law. For example, if the prompt is a \emph{plan} describing exactly what to do, then:
\begin{equation}
I_\nu (h : p) = \ell_\nu (h) - \overbrace{\ell_\nu (h|p)}^{0} = \ell_\nu (h)
\end{equation}
and we get the maximum possible speed-up, meaning that the time to execute the search becomes merely $\tau_U = T(h^*)$, the minimum possible time required by a trajectory to solve the task.

Alternatively, a prompt $p$ may not specify the whole trajectory, but all the information that it has may be relevant to the trajectory, that is $I(h : p) = \ell_\nu (p)$. In this case, we get a significant speed-up $2^{\ell_\nu (p)}$. For example, just 10 good bits of prompt (a few tokens) can reduce the time to find a solution by $\sim 1024$ times. We can think of this as a useful \emph{hint} (``\textit{Solve the following task. The following technique may be useful: ...}'') that brings down the time to solve a problem from hours to minutes.

\subsection{Example of GHC scaling: Santa Fe process}\label{sec:santa-fe}

So far we have assumed that our data generating process satisfies the Generalized Hilberg's law scaling $I(X_n : Y_m) = n^\beta + m^\beta - (n+m)^\beta$, and we anticipated that this relates to having an infinite distribution of facts that appear in the data following a long-tail distribution. Following \cite{debowski2011vocabulary}, we now explicitly construct such a process, showing that the GHC scaling definition makes sense and how exactly `facts' relate to scaling.

Let $\{Z_k\}_{k=1}^\infty \sim \text{Bern}(1/2)$ be an infinite set of binary properties which are sampled \emph{before} any text is generated. We can think of them as facts about the world, which may be referenced in the text. Importantly, since the $Z_k$ are sampled only once at the very beginning and do not change over time, once a fact is first encountered, we know its value in any future text.\footnote{An alternative view is that the process is extremely long memory: after the first time it generates a value for $Z_k$ it remembers it and reuses at all later times.} Some facts are referenced very often, others very rarely. Empirically, natural frequencies are well captured by a Zipf power law:
\begin{equation}
p(k) = c k^{-1/\beta}
\end{equation}
for some normalization factor $c$.

To generate a sequence $X$, we concatenate the index of a random fact and its value:
\[
X = ((k_1, Z_{k_1}), (k_2, Z_{k_2}), \ldots, (k_n, Z_{k_n}))
\]
where $k_i \sim p(k)$. We generate $Y$ similarly.

\begin{theorem}{Santa Fe Process GHC Scaling}{}
The Santa Fe process described above follows GHC scaling:
\[
I(X_n : Y_m) = n^\beta + m^\beta - (n+m)^\beta
\]
\end{theorem}

\begin{proof}
We now want to show that for this process $I(X_n; Y_m)$ follows GHC scaling. Let's first rewrite:
\[
I(X: Y) = H(X) + H(Y) - H(X Y).
\]
To compute the compression cost $H(X)$ of $X$, the key observation is that we only need to encode the value of a fact the first time we see it (since it remains constant). How many unique facts appear in $X$? Asymptotically, this is given by:
\[
U(n) = \sum_{k=1}^\infty [1 - (1 - p_k)^n] \approx C_{\beta} n^\beta.
\]
Using this, the compression cost $H(X)$ is the cost of encoding the $n$ random indices (using $H(p)$ bits per index), plus the cost of encoding the $U(X)$ unique properties:
\[
H(X) = n H(p) + C_{\beta} n^\beta.
\]
The cost of $H(Y)$ and $H(XY)$ are computed similarly. Putting all together we get the desired result:
\begin{align*}
I(X; Y) &= H(X) + H(Y) - H(XY)\\
&= [n H(p) + C_{\beta} n^\beta] + [m H(p) + C_{\beta} m^\beta] - [(n + m) H(p) + C_{\beta} (n + m)^\beta]\\
&= C_\beta [ n^\beta + m^\beta - (n+m)^\beta ]
\end{align*}
\end{proof}

This construction is quite artificial -- real world data are not a stream of random facts, indices, and immutable binary values. But it does highlight in a simple way some key, and more fundamental, issues: First, the process has \emph{infinite complexity} since it needs to store the value of all the $Z_k$. Recall that by \Cref{theorem:max-speed-up}, this is the exactly the setting where learning can make the most difference, since we can obtain arbitrarily large speed-ups.
Second, the distribution of facts has to follow a power law. Some facts (\emph{``the color of the sky''}) are more referenced than others (\emph{``the house number of John Doe''}) and power laws are abundant in real data and have several theoretical justifications (rich-get-richer effect, least-effort/max-entropy trade-off, monkey-typing with intermittent silences \cite{miller1957some,mandelbrot1954structure}, etc.).

In terms of reasoning traces, one may think of `facts' as \textit{functions} or \textit{theorems} that one may invoke by calling their name (index). Since functions/theorems are reused and always remain constant, after memorizing them we can significantly compress future reasoning. 
This intuition has  motivated a significant volume of research, both regarding universal solvers \cite{solomonoff1984optimum,schmidhuber1997shifting,schmidhuber2013powerplay,franz2015toward} as well as program synthesis \cite{muggleton1991inductive,liang2010learning, ellis2023dreamcoder,grand2023lilo} and 
reinforcement learning \cite{sutton1999between,lin2023program}.

\section{Inversion of Scaling Laws}

So far we have established the fact that learning from data leads to a speed-up in finding a solution to an unforeseen task. However, the equivalence
\begin{equation}
\log (\text{speed-up}) = I_\nu (h : D)
\end{equation}
tells us something stronger: we learn to be faster \emph{if and only if} we learn from data. This suggests that we can learn something from data if and only if we train with a time optimization objective.

Let's work through an example. Suppose we want to train a universal solver, and (as is natural) we use as reward function the expected number of correct solutions, determined by a function $R$, averaged over some distribution $f \sim q$ of tasks:
\begin{equation}
L = \mathbb{E}_{f\sim q} \mathbb{E}_{h \sim \nu}[R(h)].
\end{equation}
Further suppose that our agent has unlimited compute power available, so that we have no need to optimize resources over their usage. What will such a system learn?

If the distributions of task is generic enough, we know that the system has to learn to perform universal computation (\Cref{theorem:universal-computation}). But that is \emph{the only} thing that it needs to learn. Having universal computation, it can implement the basic Solomonoff-Levin Universal search algorithm, which will always find the solution to the task, thus achieving maximum reward. It will take eons to find the solution, but since compute is free for this agent, that is not a problem.

To further clarify, suppose we want to teach the model to play chess. Training is not necessary to achieve a better reward, since a standard tree-search over all the possible moves will eventually find the best move to make. Training is required only to \emph{reduce the time} that it takes to find the best move.

\begin{remark}{Only time bound systems learn}{}
If a system is not penalized for the time it takes to find a solution to the task, it is optimal to always brute-force a solution without learning anything. Vice versa, any system that optimizes time has to learn at least $I_\nu (h : D) = \log (\text{speed-up})$ bits of information from the data.
\end{remark}

Going more in depth, we can look at how we expect $I_\nu (h; D)$ to behave as we scale the model. First, note that as we scale the maximum time allowed for a trajectory, we also usually want to jointly scale the amount of weights of the model \cite{kaplan2020scaling}. So if $T$ is the maximum time for the trajectory, the number of weights will be some monotone function $|\theta| = f(M)$. Note that the number of weights puts a constraint on the maximum amount of information $I(\nu : D)$ about the data that we can store in the model parameters. But since $I_\nu (h : D) \leq I(\nu : D)$, this also puts an upper-bound on the per-trajectory information $I_\nu (h: D)$. Next, let's look at how much information the model is forced to capture if it wants to have perfect performance on the task. We need $\tau_\nu (h|D) \leq T$ so we need to store enough information to speed up the search until it takes less than $T$ total time. This means:
\begin{align}
\tau_\nu (h|D) &= \frac{\tau_\nu (h)}{2^{I(h : D)}} \leq T \\
\Rightarrow I(h:D) &\leq \log \tau_\nu (h) - \log T
\end{align}
so as expected from the discussion before, the amount of information we \emph{need} decreases as $T$ increases. Putting the two bounds together we obtain the curve for $I_\nu (h:D)$ shown in \Cref{fig:inversion-curve}.

\begin{figure}
\centering
\includegraphics[width=0.5\textwidth]{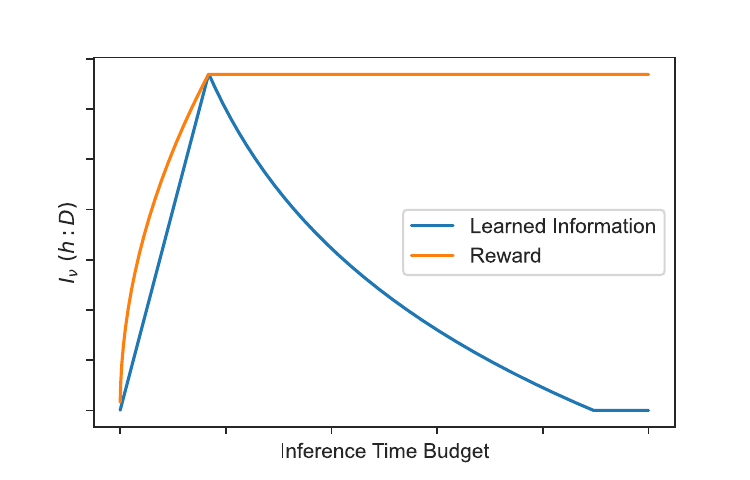}
\caption{\label{fig:inversion-curve} Simplified illustration of expected scaling law inversion curve. Consider wanting to solve a task --- initially requiring time $L$ --- in less than $T$ reasoning steps. Further assume we model capacity (number of weights) with the time budget $T$.  Initially, the model stores as much information as its weights allow to reduce the required time from $L$ to less than $T$. Once $I(h:D) = \log \frac{L}{T}$ bits are stored, he task is solved exactly (the reward is maximized). Beyond this point, increasing $T$  may actually reduce stored information, since less is needed to meet the time budget. At $T=L$, no information needs to be stored to achieve perfect reward.
}
\end{figure}

As we scale the model, when the inference time budget and the number of weights are small, the model learns as expected, acquiring information from the data, storing it into the weights. The expected reward it obtains steadily increases as more problems become solvable within the allotted time budget. At some point, the amount of information the weights can store is large enough that, thanks to the speed up, all trajectories are solvable within the time budget. At that point the reward is always optimal and stabilizes. But, paradoxically, if we further increase the time budget the model can use brute force search more, and the information it needs to acquire starts \emph{decreasing} until it reaches zero. This is the \textit{savant regime}, where the model can default to expensive brute force search with no learning whatsoever, yet still achieving optimal reward. In this regime,  optimal performance (orange curve) comes from excess capacity rather than ``insight,'' as measured by learned algorithmic information (blue curve).

To avoid entering the savant regime, one option is to penalize time using a reward like:
\begin{align}
\mathcal{R} &= R(x, y) - \lambda \log \tau_\nu (x \mapsto y)\\
&= R(x, y) + \lambda I_\nu (h : D)
\end{align}
Which forces the model to actually learn algorithmic information. 
The training of many current large reasoning models \cite{guo2025deepseek,yu2025dapo} includes similar objectives that maximize reward while minimizing inference time. The results above suggest that this regularization not only reduces inference cost---a common motivation for penalizing time---but also improves the quality of what the model learns.

Note the sharp contrast with information theoretic regularizers. In a classic machine learning setting, we want to maximize \cite{achille2018emergence}:
\begin{equation}
\mathcal{R}_{\text{reg}} = R(x,y) - \alpha I(w : D)
\end{equation}
That is, we want to \textbf{minimize} the amount of information in the weights \cite{achille2019information}. This ensures generalization, and is the basis of the Minimum Description Length (MDL) principle. For reasoning, however, we want to \textbf{maximize} the information in the weights in order to minimize time.
In transduction in a verifiable setting there is no issue of generalization, since we have access to all relevant data and a verifier can that tells us whether the task has been solved. More generally the trade-off is not trivial.

Consider a real biological agent that has to model a physical environment in order to act. The agent may learn the correct, generalizable, underlying physical rules which would be optimal from an MDL perspective. But if those rules involve a large amount of computations, the agent may not be able to use them in a reasonable time for survival. It may instead be optimal for the agent to memorize several case-by-case rules (e.g.,  a feather falls with a certain speed) even if they do not generalize (a feather doesn't fall with the same speed in a vacuum), if they allow it to quickly come up with approximate, but timely, estimates that are better for survival than the best estimate rendered belatedly.

It is common to refer to transduction as \textit{System-2} behavior and induction (restricted to a single forward pass) as \textit{System-1} behavior \cite{kahneman2011thinking}. Encoding reasoning traces into the weights can then be thought as \textit{automatization}: the common thinking pattern are made faster by moving them to System-1. Whether this is advantageous depends on environmental stability. In stationary environments, automatization allows faster reaction and better energy usage. However, in time-varying environments, the ability to reason at inference time cannot be fully replaced by a set of fixed learned behaviors. 

\section{Maximizing a Continuous Reward}
\label{sec:continuous-reward}

So far we have studied universal solvers under the assumption that the reward function $R(x, y)$ has a binary ``success/failure'' value. In the more general case, the universal solver is required to maximize a continuous reward function. If we know the maximum achievable reward $R_\text{max}$, and we know that the reward can be achieved, we can define a new binary reward function ``has achieved the maximum,'' thus reducing to the binary case. However, generally we do not know what is the achievable maximum and, more importantly, we do not know if the search cost to find a better solution will be worth the value. That is, \textit{a practical universal solver has to decide when to stop}.

The objective underlying this decision is optimization of the user's net gain. Consider a universal solver $U$ that over time outputs various candidate solutions $y_t$, and let $R^* = \max R(x, y_t)$ be the maximum reward achieved by any solution up to now. The value obtained by the user stopping at this point is:
$$J = R^* - \lambda T_{\text{total}}$$
where $R^*$ represents the best reward achieved across all explored traces, $\lambda$ is the per-token cost, and $T_\text{total}$ is the total compute time consumed until now. The question is whether investing additional compute to search for a solution is likely to improve $J$.

Deciding whether a reward function can be improved, let alone whether it is convenient to do so, is generally undecidable. Hence, we need to assume we have learned a \textit{forecasting model} $\psi_\theta$ to estimate the distribution of rewards for each potential continuation $h$, enabling principled decisions about which reasoning paths to pursue \cite{kleinman2025e1}. Under this assumption, we can formulate the search for optimal solutions as a \textit{Pandora's box problem with order constraints} \cite{weitzman1978optimal,boodaghians2020pandora}.

\textbf{Universal search as Pandora's box problem.}
We can visualize all possible computations to solve a task as a rooted tree \(\mathcal{T}\). A node $n$ represents a partial reasoning trace; exploring an extension of a reasoning trace (i.e.,  a child of $n$) for $t$ tokens incurs a cost $c =\lambda t$ and, if we terminate after the extension, yields a terminal reward
\(R\sim \psi_\theta(\cdot\mid x,z)\).
Nodes naturally obey order constraints: a child may be explored only after its parent.

Given the current best obtained reward $R^*$, the \emph{incremental value} of exploring a new child node $n$ is
\begin{equation}
\Delta(n) \;=\; \mathbb{E}\!\left[(R - R^\star)_+\right] \;-\; \lambda t.
\label{eq:delta}
\end{equation}
While we could greedily pick the next computation to perform by maximizing \(\Delta\), this generally leads to suboptimal solutions. Searching for the best possible strategy is, in principle, exponentially complex. However, Weitzman \cite{weitzman1978optimal} showed that a simple greedy strategy does exist to optimally solve this problem. In particular, Weitzman defines a \emph{conditional reservation value} for each candidate extension as the unique $z_n$ solving
\begin{equation}
\mathbb{E}\!\left[(R - z_n)_+\right] \;=\; \lambda t.
\label{eq:reservation}
\end{equation}
The resulting value of $z_n$ is called the \textit{Gittins index} of the node. The provably optimal policy is then to visit at each step the unexplored node which has the currently highest Gittins index. It is instead optimal to \textit{terminate} the search when:
\begin{equation}
R^\star \;\ge\; \max_{n\in \text{unexplored}} z_n
\label{eq:stop}
\end{equation}
at which point there is no node that we can visit that is expected to improve the final objective $J$. That is, the expected improvement in reward does not compensate the cost of exploring the node.

This framework can be extended \cite{boodaghians2020pandora} to the case when there are constraints on the order of opening the boxes (e.g.,  testing a solution obtained after 1,024 thinking tokens requires first reaching 512 tokens), in which case the value of a box also needs to take into account the value of the boxes it allows access to.

\textbf{Making decisions with Gittins Indexes.}
Once we use the Gittins index $z_i$ using our forecasting model, we have a simple criterion to make several key decisions during our search. For example, $z_i$ allows us to decide: (i) when to \textit{continue extending a reasoning trace} (if the Gittins index of a child is the highest); (ii) when to branch a trace (if the child of a parent node achieves the highest index), in particular when it is optimal to restart from scratch by sampling a new reasoning trace), and (iii) when it is optimal to stop attempting to improve the solution of a problem and return the current best (when no node has a better index than the current reward).

\textbf{Need for a Forecasting Model.} Exploring greedily based on Gittins indexes remain optimal as long as: (i) the reward distribution is known in advance, and (ii) the distribution of different boxes are independent (the reward observed for one box does not affect the predicted reward on others). When this is not the case, using Gittins indices may not be optimal, but generally remains a strong policy \cite{scully2025gittins}.

In general, however, we do not know \textit{a priori} the distribution of possible rewards we may obtain. For example, we can't know in advance %
that thinking for 1,000 or 10,000 tokens will lead to finding a correct solution. In order to efficiently optimize the net reward $J$ an agent also has to learn to \textit{forecast} both the cost of an exploration attempt, and the probability of it improving over the current best solution \cite{kleinman2025e1}.

\section{Maximalistic Models of Computation}
\label{sec:maximalistic}

The scaling law for speed discussed in \Cref{sec:scaling-laws} describes the gains obtainable by assigning higher probability to common patterns learned from past reasoning traces. However, this is not yet the maximum speed-up we can achieve. Consider for example a universal task solver that is frequently asked to solve tasks requiring the computation of a Fast-Fourier Transform (FFT). Once it learns from past data the correct algorithm to compute an FFT, it can instantiate that procedure each time, significantly reducing exploration cost. However, it still must execute hundreds of algorithmic steps -- thousands of ``thinking tokens'' for an LLM -- to reach a solution. This is far from optimal: a better solution would be to compute the FFT in a single step, either by modifying the weights to implement the relevant circuitry or by invoking a specialized tool.

Two forces drive such further speed gains: \textit{automatization} and \textit{abstraction}. In automatization, as we have discussed, skills migrate from costly System-2 processes to fast System-1 routines. Instead of re-deriving physical laws, the model learns stable approximations of relevant dynamics; instead of re-searching for an FFT, it executes a compact surrogate. In abstraction, long chains of reasoning are compressed into atomic operations.
For example, rather than re-proving a property whenever needed, the model introduces a lemma --- an invocable unit that replaces a lengthy sub-derivation.

Both phenomena can be understood as expansions of the model’s effective instruction set. In a CPU, each executed instruction modifies the internal state, and a carefully constructed minimal instruction set (RISC)\footnote{Reduced Instruction-Set Computer} suffices for universality. However, some long patterns of operations may appear frequently. This motivates the introduction of complex instructions (CISC) that, even if technically redundant,  can perform the same long state transition in a single operation, trading additional real estate (e.g., area) for reduced latency. That is, space redundancy is used to buy time minimality. Similarly, each generated token in a reasoning LLM updates its internal state. The system may be universal even with simple states and transitions. Yet, instead of emitting hundreds of primitive tokens, we can introduce specialized tokens/dynamics -- or dedicated callable tools -- that let the model `jump' directly to the same final state without having to run through the individual steps.

This introduces a familiar trade-off. Adding tokens, tools, or weight-level circuitry increases parameter count and training complexity, and raises orchestration and verification burden. At the same time, it can markedly reduce wall-clock latency, exploration cost, and the memory pressure of long reasoning traces. For modern LLMs, the balance differs from classical CPUs: scaling parameter count is comparatively straightforward and efficient, while long sequential chains of thought are expensive, hard to parallelize, and brittle with respect to context-length and recall.

These constraints suggest a natural design pressure toward \textit{maximalistic models of computation}: computational engines equipped with rich tool libraries, learned subcircuits for frequent subproblems, and token types representing complex operations. Concretely, this includes (i) a library of skills—APIs, theorem banks, solvers—with learnable dispatch; (ii) token or type extensions that encode compound operators; and (iii) fine-tuned subnets that implement high-value routines (e.g.,  parsing, algebraic transforms, approximate simulation). Training signals should reward short-horizon solutions and penalize unnecessary long-form reasoning when a reliable jump exists.

Maximalistic models of computation stand in contrast with \textit{minimalistic} ones, such as Turing Machines, which are designed to operate with the smallest possible instruction set on the smallest possible dictionary through a computer with the smallest number of components. Such simplicity makes sense if the goal is analysis by humans. However, beyond interpretability, this pressure towards simplicity does not reflect the structure of the world that an AI Agents must interact with. Such world instead engenders pressure towards speed, which can be optimized by increased complexity.

Maximalism is not mere memorization. Automatized routines and abstracted lemmas are structured compressions of procedures, not rote patterns. They improve reliability and latency on recurring structures (FFT, parsing, algebraic simplifications), while remaining fallible on rare, non-stationary, or adversarial cases where general reasoning must reassert itself. Selecting which operations deserve ``instruction status'' ---  balancing parameter growth against step reduction --- and validating the safety of state jumps is a key problem for reasoning LLMs.

As minimizing reasoning-trace length becomes the dominant objective, we should expect models to shift toward maximalistic computation, mirroring the biological transition from System-2 deliberation to System-1 fluency. The path to greater speed is not only better search, but better \textit{steps}.

\section{Discussion and related prior work}
\label{sec:discussion}

The subject matter of this paper falls within the scope of statistical machine learning. Traditionally, one starts with instantiated data which implicitly define \textit{the task}, and arrives at a model that performs the same inference computation on all future instances of the same task. However, once trained models are used as generative distributions, as customary in generative AI, they exhibit behaviors that were not explicitly encoded in the training data nor the chosen loss. A key such behavior is the ability to perform variable-length inference computation, which increasingly often leads to solving previously unseen tasks. In classical (inductive) machine learning, there is no feedback mechanism at inference time, so one can only evaluate the quality of a model \textit{post-hoc}, typically on data other than the ones at hand. Agents, on the other hand, interact with the environment, which provides feedback, and/or can call tools to solicit feedback or spawn simulations at inference time prior to rendering an outcome.
Inference computation can therefore adapt depending on the resulting feedback. This mode of interaction calls for a different approach to learning, which aims to empower \textit{transductive inference}. The power of LLMs stems from the fact that, despite being trained \textit{inductively} (despite the name, so-called unsupervised pre-training is simply next-token prediction, is a standard multi-class supervised classification problem), they operate transductively leveraging their chain-of-thought. In this paper, we explore the foundational principles of such \textit{transductive learning}, and its limits, including bounds and power laws. 

\textbf{Transduction, In-context Learning, and Solomonoff Inference.}
Transduction in the form of learning jointly from labeled training examples and unlabeled test samples was championed by Vapnik \cite{chapelle1999transductive, vapnik2006transductive, gammerman2013learning}.

The dichotomy we present between transductive and inductive inference has long been considered in machine learning under various guises, such as the distinction between learning to generalize'' and learning to reason'', between System~1 and System~2 \cite{kahneman2011thinking,geffner2018model}, between \textit{fluid intelligence} and \textit{crystallized intelligence} \cite{hernandez2017measure}, or between learners and solvers \cite{geffner2018model}. Our objective is not to further justify the need for transduction or reasoning, which has amply been done, but rather to contextualize it through the lens of LLMs as computational engines and to provide scaling laws connecting time-to-solve with the learning of algorithmic information.

An early observation was that language models can learn multiple tasks implicitly through the unsupervised language modelling objective \cite{radford2019language}, and exhibit diverse behaviors when adequately prompted.  In-context learning \cite{dong2022survey}, which is a form of transductive inference, introduces demonstration examples into the model context to elicit desired behavior.  It has been shown that LLMs can perform optimization algorithms such as gradient descent and and ridge regression transductively at inference time from in-context examples  \cite{akyurek2022learning}.  \cite{garg2022can} demonstrates that sparse linear functions, decision trees, and two-layer networks can be learned in-context.    \cite{wu2023many} investigates what minimal pretraining is necessary to induce in-context learning, showing that a small pretrained model can achieve close to the Bayes optimal algorithm.  \cite{li2023transformers} introduce a theory of in-context learning where a hypothesis is formed at inference time, obtaining generalization bounds.  \cite{JMLR:v25:23-1042} demonstrates that a single self-attention layer trained by gradient flow to perform in-context learning converges to a global minimum that is competitive with the best predictor on the test distribution.  The connection between in-context learning and Solomonoff inference was identified in \cite{grau2024learning}, where the authors attempt to learn the Solomonoff semimeasure directly by sampling programs and training on their outputs.  In \cite{deletang2023language}, motivated by the inductive theory of Solomonoff \cite{solomonoff1964formal1,solomonoff1964formal2,solomonoff1999two}, the authors demonstrate that LLMs can outperform general purpose compressors, even for audio and visual data.  The connection between Solomonoff induction and neural network optimization as a form of program search was mentioned earlier in \cite{kaiser2015neural}.

As we have noted in the introduction, time plays no role in Solomonoff Inference, nor in in-context learning. Neither involve actual ``learning'' in the classical inductive sense: The weights are fixed and the same task, presented in-context multiple times, requires repeating the same effort to yield no different outcome each time. However, time plays a key role in \textit{learning transduction}, which is the core motivation of this work.

\textbf{Systematicity of Dichotomies.}
Several conceptual pairs recur throughout this work (\Cref{tab:dichotomies}), which can, among other interpretations, be seen as projections of a the distinction between inductive and transductive inference onto different aspects of a learning system.

In inductive inference, the goal is to compress past observations into a fixed map that generalizes to similar future inputs. The bottleneck is representational capacity (space/parameters); success is measured by accuracy on held-out data; simplicity aids generalization; and memorization of training-specific detail is harmful. Computation is a single feedforward pass (System~1), interpolating among previously seen patterns. Space-pressure (amount of information stored) is important to learn generalizable features rather than  memorizing past results. A prototypical task is language modeling: reduce uncertainty on the next token given the statistical structure of past text.

In transductive inference, the goal is to reason about a \textit{specific} instance at inference time, guided by a verifier or reward signal. The bottleneck is computational budget (time/tokens); success is measured by speed to a verified solution; and memorization of reusable algorithmic structure is beneficial (\Cref{theorem:information-speed}). Computation is variable-length chain-of-thought (System~2), extrapolating beyond the training distribution to solve novel tasks. Time pressure is essential for learning to occur (\Cref{theorem:learning-time}) and avoid savantry. A prototypical task is theorem proving: search for a proof that satisfies a checker. 

\begin{table}
\centering
\small
\renewcommand{\arraystretch}{1.25}
\begin{tabular}{@{}l||l|l@{}}
\toprule
\textbf{Mode of inference} 
& \textbf{Inductive} & \textbf{Transductive} \\
\midrule
{\bf Bottleneck} & Space (parameters) & Time (tokens/compute) \\
{\bf Transferability} & Interpolation & Extrapolation \\
{\bf Memorization} & Harmful (overfitting) & Beneficial (speed-up) \\
{\bf Inference mechanism} & System~1 (fixed pass) & System~2 (variable reasoning) \\
{\bf Sample task} & Language modeling & Theorem proving \\
{\bf Scaling } & Accuracy vs.\ size & Speed vs.\ size \\
{\bf Information objective} & \underline{Minimize} $I(w : D)$ & \underline{Maximize} $I(h : D)$ \\
\bottomrule
\end{tabular}
\caption{Dichotomies outlined in this work can be viewed as projections of the induction--transduction distinction. Each row contrasts the same system property under the two regimes.}
\label{tab:dichotomies}
\end{table}

The two regimes are not mutually exclusive. Automatization (\Cref{sec:maximalistic}) moves frequent reasoning patterns from the transductive to the inductive regime, trading space for time. Conversely, when the environment changes, previously automatized behaviors must be re-derived transductively. A complete agent must balance both, and the optimal operating point depends on the stability of the environment and the cost of compute.

{\bf LRMs, SLMs, VLMs, VLAs, World Models, etc. (nomenclature).}
The term LLM originally referred to large-scale Transformer-based models (pre-)trained as next-token predictors using large-scale corpora of natural language, then co-opted as probability distribution to sample new natural language text. Optionally, these models could be fine-tuned by scoring such generated expressions using human preference, an external reward mechanism, or by the model itself through self-assessment. It is also common to call the outcome of the same process a `World Model' (WM) if trained on sensory data such as video or audio instead of natural language, or  `vision-language model' (VLM) if trained on both, or 'vision-language-action' model if the output expression is used to issue commands to an actuator, or `large reasoning models' (LRMs) if they are used to generate variable-length trajectories prior to rendering the decision or action. In our nomenclature, any large-scale predictor trained on sequential data with latent logical/linguistic structure (with objects, relations, functions, etc.) develops an inner state space with an internal ``Neuralese language'' \cite{trager2023linear}. Sensory data are replete with latent discrete entities \cite{ayvaci2011detachable}, their relations (topological, geometric, dynamic, semantic) and (de)composition into parts (meronomies) or abstract concepts (taxonomies). In our definition of LLM, therefore, where `language' is not restricted to natural language and instead refers to any form of Neuralese, VLMs, WMs, VLAs, LRMs and other variants are also LLMs. We also include in the term LLMs models that use different architectures, so long as they have a `state' (memory), whether explicit (as in state-space models) or implicit by co-opting a sliding window of data, as in autoregressive Transformers \cite{zancato2024b}. Since the largest LLMs at this point comprise trillions of parameters, some now refer to models with merely billions of parameters as `small language models' or SLM. Obviously, `small' is subjective, and these models have no architecture, structural, functional, or conceptual difference from their `large' counterpart, so they too are just LLMs. Empirically some emergent phenomena are only observed at scale, but this does not mean that there is a clear divider between `large' and `small', even phenomenologically since smaller models can still be distilled from larger ones and maintain their behavior even if it would not have emerged from cold-start using the same training protocol \cite{guo2025deepseek}.

{\bf Embodied AI.} The results described in this paper pertain to both software agents that exist within the world of bits, as well as embodied agents that interact with the physical environment. While this may seem counter to the forced dichotomy between LLMs and so-called World Models, once the sensory data is tokenized the two become the same mathematically. Regardless of how a model is trained inductively, once it acquires the ability to perform transductive inference, it can act as an agent. This could be in the world of bits, where interaction with the surrounding environment is through APIs and function calls, or in the world of atoms, where sensors provide measurements that the agent turns into a representation of the environment (which is an abstract concept finitely encoded in the state of the agent \cite{achille2022learnability}) and operate on it (i.e.,  reason), to command physical actuators that affect the  the environment. Such an environment then provides a feedback signal, ultimately in the form of ``verification'' (e.g.,  survival or rewards). The reasoning agents exists in a finitely-encoded world, and interfaces with the physical world through encoders and decoders. The core of all these agents is the ability to perform transductive inference within the discrete/discretized representation, which requires computation as described in this document. While evolution proves that processing sensory data is sufficient to foster the emergence of reasoning, language is already conveniently distilled (symbolized and compressed) making the traversal of the evolutionary path unnecessary for the emergence of reasoning. In this sense, agentic AI subsumes embodied AI, where the latter focuses on the source of the data (sampled physical sensory measurements) and focuses on physical actuators commands (physical action). 

{\bf Universal Computation, Universal Search, Algorithmic Information.}  Theoretical Computer Science has devoted decades to the development of universal algorithms; indeed, Levin's paper that introduced his universal search algorithm seeded a large portion of the subsequent literature on computational complexity theory. Since we only use the concepts driving universal search, we do not review this sizable body of work here and refer the reader to any textbook on complexity theory. One exception we make is to comment on the literature of Kolmogorov Complexity and Algorithmic Information Theory, which is thoroughly covered in textbooks \cite{li2008introduction}. 

While Kolmogorov's theory is useful for certain asymptotic analysis, and we make heavy use of it in this manuscript, it is worth pointing out that it fails its original intent to canonically separate ``information'' from ``randomness' with the concept of the Structure Function. In our approach, the information in a trained model is specific to the model, hence subjective.

{\bf Reinforcement Learning.} Our formulation---searching for $y^*(x) = \arg \max_y f(x,y)$ given a score function $f$ --- superficially resembles reinforcement learning (RL), where an agent maximizes cumulative reward through interaction with an environment. However, the goals differ in a key way. Standard RL seeks to learn a fixed \textit{policy} for a \textit{particular} environment. We instead want a universal solver to craft a \textit{new} policy for tasks no agent has previously encountered. Put differently, a particular RL problem is one possible task $f$ that our universal solver could be asked to address, but the solver itself is not specialized to any single such task. While RL considers time-penalized objectives, the time RL tries to minimize is the number of actions or observations, while we try to minimize the compute time used by the agent to \textit{decide} which actions to take or to process the observation. The theoretical questions we focus on---how learning compresses the multiplicative constant $2^{\ell(A)}$ of universal search, and how scaling laws govern the resulting speed-up---are orthogonal to the design of any particular RL algorithm, which is why we do not review that extensive literature here.

{\bf Markov Chains.} We use general dynamical systems as a model of computation, but arguably the most important piece is the transition probability $\nu(s_{t+1} | s_t)$, which defines a Markov Chain. One may wonder why we need to introduce such general machinery, including `proper time,' when standard concepts from Markov Chains, such as the expected hitting time, would suffice.
As noted, however, hitting time could be made arbitrarily small or large without fundamentally changing the computations performed. 
We also note that since AI Agents interact with the unknown environment, they are not closed systems describable with a Markov chain, but can still be described as a dynamical system.
We also note that our use of dynamical systems or Markov structure is restricted to modeling the \textit{computations of the agent} solving the task, not on the data itself. In fact, we stay clear of ever making any assumption on the data generating distribution aside from it satisfying the Hilberg's law, which defies the Markov assumption. We also note that, while ultimately one could argue that any physical process has a latent finite-memory or Markov structure, there is a fundamental difference between a Markov process of {\em known} order, and one of finite but unknown order. In the former case, one can just instantiate a model with sufficient capacity and know that it will learn the statistics of all order. In the latter case, one can never know when they have observed enough, and must instead be reeady to add to memory -- effectively treating the process as non-Markov. The theory described in this manuscript pertains to the latter case.

\textbf{Memorization and Generalization.}
Information complexity based generalization theory formalizes the notion that generalization occurs whenever the information the learned hypothesis contains about the training data is minimal (low memorization).  \cite{banerjee2021information} demonstrated that a single information exponential inequality is sufficient to derive PAC Bayes bounds \cite{catoni2003pac}, the mutual information bound \cite{xu2017information}, and a version of the conditional mutual information bound \cite{pmlr-v125-steinke20a}.  Even the classical finite-hypothesis and VC dimension bound \cite{vapnik2015uniform} can be viewed as primitive versions of such bounds.  All the aforementioned results assume that the training and test data are drawn as i.i.d. samples from a common distribution.  In contrast with the aforementioned theory, it has been demonstrated that there are learning tasks where memorization is provably necessary to achieve high accuracy \cite{brown2021memorization}, and that mitigating memorization can cause the model to fail on long-tailed tasks \cite{feldman2020neural, feldman2020does}.  There is evidence that natural language is akin to such long-tailed tasks that require memorization.  Recent work demonstrate that there are models for language which explicitly memorize the training data that perform well.  \cite{khandelwal2019generalization} introduce nearest-neighbor language models (kNN-LM) which predict the next token according to the $k$ nearest neighbors of the context embedding, which requires explicitly encoding all context embeddings and their subsequent token. \cite{min2023silo} demonstrates that augmenting parameteric LLMs with a kNN-LM can significantly boost performance.  Even stronger,  \cite{liu2024infini} demonstrates that a generalization of an $n$-gram model (dubbed $\infty-$gram) outperforms kNN-LM, while losslessly encoding the training data into a suffix array data structure.

{\bf Additional Related Work.} 
\cite{li1992inductive} championed the analysis of inductive learning from the lens of algorithmic complexity.  \cite{hutter2005universal} provides a comprehensive view of the compression approach to learning algorithmic information, complementary to ours.

\cite{schmidhuber2015learning} develops reinforcement learning algorithms with insights from algorithmic information theory. \cite{schmidhuber1997discovering} uses Levin search to find \textit{low-complexity solutions} to a classification task, whereas we use Levin search to define a notion of time for universal task solvers.
\cite{schmidhuber1997shifting}
notes that solutions generally share zero algorithmic information and highlights the importance of learning for decreasing time. They also use an adaptive Levin search with information from the past, although the final solution bears no similarity to our approach. \cite{schmidhuber2013powerplay} applies ideas from active learning to general problem solving. 

Finally, the study of dynamical systems as universal computer was developed extensively in the context of hybrid (continuous/discrete) systems, as exemplified in \cite{branicky1995universal}. Although the type of dynamical systems we describe here evolve in discrete space and time, some of the themes are recurrent and additional insights may be gleaned from revisiting that literature.

\subsection*{Acknowledgments}

Various aspects of the content of this manuscript has been discussed with multiple people during the period of 2021-2024. In particular we thank Ben Bowman, Matthew Trager, Leonid Levin for discussions, and Jurgen Schmidhuber, George Pappas for pointing to references we missed in an earlier version of this manuscript. Paulo Tabuada provided multiple comments and suggested corrections.

\bibliographystyle{plain}
\bibliography{biblio}

\begin{thebibliography}{10}

\bibitem{achille2019information}
Alessandro Achille, Giovanni Paolini, and Stefano Soatto.
\newblock Where is the information in a deep neural network?
\newblock {\em arXiv preprint arXiv:1905.12213}, 2019.

\bibitem{achille2018emergence}
Alessandro Achille and Stefano Soatto.
\newblock Emergence of invariance and disentanglement in deep representations.
\newblock {\em Journal of Machine Learning Research}, 19(50):1--34, 2018.

\bibitem{achille2022learnability}
Alessandro Achille and Stefano Soatto.
\newblock On the learnability of physical concepts: Can a neural network
  understand what's real?
\newblock {\em arXiv preprint arXiv:2207.12186}, 2022.

\bibitem{akyurek2022learning}
Ekin Akyurek, Dale Schuurmans, Jacob Andreas, Tengyu Ma, and Denny Zhou.
\newblock What learning algorithm is in-context learning? investigations with
  linear models.
\newblock {\em arXiv preprint arXiv:2211.15661}, 2022.

\bibitem{aschenbrenner2024situational}
Leopold Aschenbrenner.
\newblock Situational awareness.
\newblock {\em The decade ahead. situational-awareness. ai}, 2024.

\bibitem{ayvaci2011detachable}
Alper Ayvaci and Stefano Soatto.
\newblock Detachable object detection: Segmentation and depth ordering from
  short-baseline video.
\newblock {\em IEEE transactions on pattern analysis and machine intelligence},
  34(10):1942--1951, 2011.

\bibitem{banerjee2021information}
Pradeep~Kr. Banerjee and Guido Montúfar.
\newblock Information complexity and generalization bounds.
\newblock In {\em 2021 IEEE International Symposium on Information Theory
  (ISIT)}, pages 676--681, 2021.

\bibitem{blumer1987occam}
Anselm Blumer, Andrzej Ehrenfeucht, David Haussler, and Manfred~K Warmuth.
\newblock Occam's razor.
\newblock {\em Information processing letters}, 24(6):377--380, 1987.

\bibitem{boodaghians2020pandora}
Shant Boodaghians, Federico Fusco, Philip Lazos, and Stefano Leonardi.
\newblock Pandora's box problem with order constraints.
\newblock In {\em Proceedings of the 21st ACM Conference on Economics and
  Computation}, pages 439--458, 2020.

\bibitem{borji2025illusion}
Ali Borji, Ganesha Damaraju, and Anusha Madan~Gopal.
\newblock The illusion of computation: Why llms are not universal turing
  machines.
\newblock {\em Available at SSRN 5584490}, 2025.

\bibitem{branicky1995universal}
Michael~S Branicky.
\newblock Universal computation and other capabilities of hybrid and continuous
  dynamical systems.
\newblock {\em Theoretical computer science}, 138(1):67--100, 1995.

\bibitem{brown2021memorization}
Gavin Brown, Mark Bun, Vitaly Feldman, Adam Smith, and Kunal Talwar.
\newblock When is memorization of irrelevant training data necessary for
  high-accuracy learning?
\newblock In {\em Proceedings of the 53rd annual ACM SIGACT symposium on theory
  of computing}, pages 123--132, 2021.

\bibitem{catoni2003pac}
Olivier Catoni.
\newblock A pac-bayesian approach to adaptive classification.
\newblock {\em preprint}, 840:2, 2003.

\bibitem{chapelle1999transductive}
Olivier Chapelle, Vladimir Vapnik, and Jason Weston.
\newblock Transductive inference for estimating values of functions.
\newblock {\em Advances in Neural Information Processing Systems}, 12, 1999.

\bibitem{cover1991information}
Thomas~M Cover and Joy~A Thomas.
\newblock Information theory and statistics.
\newblock {\em Elements of information theory}, 1(1):279--335, 1991.

\bibitem{debowski2011vocabulary}
Lukasz Debowski.
\newblock On the vocabulary of grammar-based codes and the logical consistency
  of texts.
\newblock {\em IEEE Transactions on Information Theory}, 57(7):4589--4599,
  2011.

\bibitem{debowski2012mixing}
Lukasz Debowski.
\newblock Mixing, ergodic, and nonergodic processes with rapidly growing
  information between blocks.
\newblock {\em IEEE Transactions on Information Theory}, 58(6):3392--3401,
  2012.

\bibitem{debowski2015relaxed}
Lukasz Debowski.
\newblock The relaxed hilberg conjecture: A review and new experimental
  support.
\newblock {\em Journal of Quantitative Linguistics}, 22(4):311--337, 2015.

\bibitem{debowski2020information}
Lukasz Debowski.
\newblock {\em Information theory meets power laws: Stochastic processes and
  language models}.
\newblock John Wiley \& Sons, 2020.

\bibitem{dkebowski2023multiperiodic}
{\L}ukasz D{\k{e}}bowski.
\newblock Multiperiodic processes: Ergodic sources with a sublinear entropy.
\newblock {\em arXiv preprint arXiv:2302.09049}, 2023.

\bibitem{dkebowski2025zipf}
{\L}ukasz D{\k{e}}bowski.
\newblock From zipf's law to neural scaling through heaps' law and hilberg's
  hypothesis.
\newblock {\em arXiv preprint arXiv:2512.13491}, 2025.

\bibitem{deletang2023language}
Gregoire Deletang, Anian Ruoss, Paul-Ambroise Duquenne, Elliot Catt, Tim
  Genewein, Christopher Mattern, Jordi Grau-Moya, Li~Kevin Wenliang, Matthew
  Aitchison, Laurent Orseau, et~al.
\newblock Language modeling is compression.
\newblock {\em arXiv preprint arXiv:2309.10668}, 2023.

\bibitem{dong2022survey}
Qingxiu Dong, Lei Li, Damai Dai, Ce~Zheng, Zhiyong Wu, Baobao Chang, Xu~Sun,
  Jingjing Xu, and Zhifang Sui.
\newblock A survey on in-context learning.
\newblock {\em arXiv preprint arXiv:2301.00234}, 2022.

\bibitem{ellis2023dreamcoder}
Kevin Ellis, Lionel Wong, Maxwell Nye, Mathias Sable-Meyer, Luc Cary, Lore
  Anaya~Pozo, Luke Hewitt, Armando Solar-Lezama, and Joshua~B Tenenbaum.
\newblock Dreamcoder: growing generalizable, interpretable knowledge with
  wake--sleep bayesian program learning.
\newblock {\em Philosophical Transactions of the Royal Society A: Mathematical,
  Physical and Engineering Sciences}, 381(2251), 2023.

\bibitem{feldman2020does}
Vitaly Feldman.
\newblock Does learning require memorization? a short tale about a long tail.
\newblock In {\em Proceedings of the 52nd Annual ACM SIGACT Symposium on Theory
  of Computing}, pages 954--959, 2020.

\bibitem{feldman2020neural}
Vitaly Feldman and Chiyuan Zhang.
\newblock What neural networks memorize and why: Discovering the long tail via
  influence estimation.
\newblock {\em Advances in Neural Information Processing Systems},
  33:2881--2891, 2020.

\bibitem{filan2016loss}
Daniel Filan, Jan Leike, and Marcus Hutter.
\newblock Loss bounds and time complexity for speed priors.
\newblock In {\em Artificial Intelligence and Statistics}, pages 1394--1402.
  PMLR, 2016.

\bibitem{franz2015toward}
Arthur Franz.
\newblock Toward tractable universal induction through recursive program
  learning.
\newblock In {\em International Conference on Artificial General Intelligence},
  pages 251--260. Springer, 2015.

\bibitem{gammerman2013learning}
Alex Gammerman, Volodya Vovk, and Vladimir Vapnik.
\newblock Learning by transduction.
\newblock {\em arXiv preprint arXiv:1301.7375}, 2013.

\bibitem{garg2022can}
Shivam Garg, Dimitris Tsipras, Percy~S Liang, and Gregory Valiant.
\newblock What can transformers learn in-context? a case study of simple
  function classes.
\newblock {\em Advances in Neural Information Processing Systems},
  35:30583--30598, 2022.

\bibitem{geffner2018model}
Hector Geffner.
\newblock Model-free, model-based, and general intelligence.
\newblock {\em arXiv preprint arXiv:1806.02308}, 2018.

\bibitem{grand2023lilo}
Gabriel Grand, Lionel Wong, Maddy Bowers, Theo~X Olausson, Muxin Liu, Joshua~B
  Tenenbaum, and Jacob Andreas.
\newblock Lilo: Learning interpretable libraries by compressing and documenting
  code.
\newblock {\em arXiv preprint arXiv:2310.19791}, 2023.

\bibitem{grau2024learning}
Jordi Grau-Moya, Tim Genewein, Marcus Hutter, Laurent Orseau, Gregoire
  Deletang, Elliot Catt, Anian Ruoss, Li~Kevin Wenliang, Christopher Mattern,
  Matthew Aitchison, et~al.
\newblock Learning universal predictors.
\newblock {\em arXiv preprint arXiv:2401.14953}, 2024.

\bibitem{guo2025deepseek}
Daya Guo, Dejian Yang, Haowei Zhang, Junxiao Song, Ruoyu Zhang, Runxin Xu,
  Qihao Zhu, Shirong Ma, Peiyi Wang, Xiao Bi, et~al.
\newblock Deepseek-r1: Incentivizing reasoning capability in llms via
  reinforcement learning.
\newblock {\em arXiv preprint arXiv:2501.12948}, 2025.

\bibitem{hernandez2017measure}
Jos{\'e} Hern{\'a}ndez-Orallo.
\newblock {\em The measure of all minds: evaluating natural and artificial
  intelligence}.
\newblock Cambridge University Press, 2017.

\bibitem{hilberg1990bekannte}
Wolfgang Hilberg.
\newblock Der bekannte grenzwert der redundanzfreien information in texten-eine
  fehlinterpretation der shannonschen experimente?
\newblock {\em Frequenz}, 44(9-10):243--248, 1990.

\bibitem{hutter2005universal}
Marcus Hutter.
\newblock {\em Universal artificial intelligence: Sequential decisions based on
  algorithmic probability}.
\newblock Springer Science \& Business Media, 2005.

\bibitem{kahneman2011thinking}
Daniel Kahneman.
\newblock {\em Thinking, fast and slow}.
\newblock macmillan, 2011.

\bibitem{kaiser2015neural}
Lukasz Kaiser and Ilya Sutskever.
\newblock Neural gpus learn algorithms.
\newblock {\em arXiv preprint arXiv:1511.08228}, 2015.

\bibitem{kaplan2020scaling}
Jared Kaplan, Sam McCandlish, Tom Henighan, Tom~B Brown, Benjamin Chess, Rewon
  Child, Scott Gray, Alec Radford, Jeffrey Wu, and Dario Amodei.
\newblock Scaling laws for neural language models.
\newblock {\em arXiv preprint arXiv:2001.08361}, 2020.

\bibitem{khandelwal2019generalization}
Urvashi Khandelwal, Omer Levy, Dan Jurafsky, Luke Zettlemoyer, and Mike Lewis.
\newblock Generalization through memorization: Nearest neighbor language
  models.
\newblock In {\em International Conference on Learning Representations}, 2019.

\bibitem{kleinman2025e1}
Michael Kleinman, Matthew Trager, Alessandro Achille, Wei Xia, and Stefano
  Soatto.
\newblock e1: Learning adaptive control of reasoning effort.
\newblock {\em arXiv preprint arXiv:2510.27042}, 2025.

\bibitem{levin1973universal}
Leonid~A Levin.
\newblock Universal sequential search problems.
\newblock {\em Problems of information transmission}, 9(3):265--266, 1973.

\bibitem{vitanyi_algorithmic}
Ming Li and Paul Vitanyi.
\newblock {\em An Introduction to Kolmogorov Complexity and Its Applications}.
\newblock Springer Publishing Company, Incorporated, 4th edition, 2019.

\bibitem{li2008introduction}
Ming Li, Paul Vitanyi, et~al.
\newblock {\em An introduction to Kolmogorov complexity and its applications},
  volume~3.
\newblock Springer, 2008.

\bibitem{li1992inductive}
Ming Li and Paul~MB Vitanyi.
\newblock Inductive reasoning and kolmogorov complexity.
\newblock {\em Journal of Computer and System Sciences}, 44(2):343--384, 1992.

\bibitem{li2023transformers}
Yingcong Li, Muhammed~Emrullah Ildiz, Dimitris Papailiopoulos, and Samet Oymak.
\newblock Transformers as algorithms: Generalization and stability in
  in-context learning.
\newblock In {\em International Conference on Machine Learning}, pages
  19565--19594. PMLR, 2023.

\bibitem{liang2010learning}
Percy Liang, Michael~I Jordan, and Dan Klein.
\newblock Learning programs: A hierarchical bayesian approach.
\newblock In {\em ICML}, volume~10, pages 639--646, 2010.

\bibitem{lin2023program}
Yu-An Lin, Chen-Tao Lee, Guan-Ting Liu, Pu-Jen Cheng, and Shao-Hua Sun.
\newblock Program machine policy: Addressing long-horizon tasks by integrating
  program synthesis and state machines.
\newblock {\em arXiv preprint arXiv:2311.15960}, 2023.

\bibitem{liu2024infini}
Jiacheng Liu, Sewon Min, Luke Zettlemoyer, Yejin Choi, and Hannaneh Hajishirzi.
\newblock Infini-gram: Scaling unbounded n-gram language models to a trillion
  tokens.
\newblock {\em arXiv preprint arXiv:2401.17377}, 2024.

\bibitem{luby1993optimal}
Michael Luby, Alistair Sinclair, and David Zuckerman.
\newblock Optimal speedup of las vegas algorithms.
\newblock {\em Information Processing Letters}, 47(4):173--180, 1993.

\bibitem{mandelbrot1954structure}
Benoit Mandelbrot.
\newblock Structure formelle des textes et communication: Deux {\'e}tudes par.
\newblock {\em Word}, 10(1):1--27, 1954.

\bibitem{merrill2023expressive}
William Merrill and Ashish Sabharwal.
\newblock The expressive power of transformers with chain of thought.
\newblock {\em arXiv preprint arXiv:2310.07923}, 2023.

\bibitem{miller1957some}
George~A Miller.
\newblock Some effects of intermittent silence.
\newblock {\em The American journal of psychology}, 70(2):311--314, 1957.

\bibitem{min2023silo}
Sewon Min, Suchin Gururangan, Eric Wallace, Weijia Shi, Hannaneh Hajishirzi,
  Noah~A Smith, and Luke Zettlemoyer.
\newblock Silo language models: Isolating legal risk in a nonparametric
  datastore.
\newblock In {\em The Twelfth International Conference on Learning
  Representations}, 2023.

\bibitem{muggleton1991inductive}
Stephen Muggleton.
\newblock Inductive logic programming.
\newblock {\em New generation computing}, 8(4):295--318, 1991.

\bibitem{orseau2018single}
Laurent Orseau, Levi Lelis, Tor Lattimore, and Theopane Weber.
\newblock Single-agent policy tree search with guarantees.
\newblock {\em Advances in Neural Information Processing Systems}, 31, 2018.

\bibitem{papageorgioucan}
Vasilis Papageorgiou, Jy-yong Sohn, Kangwook Lee, and Dimitris Papailiopoulos.
\newblock Can general-purpose language models emulate a general-purpose
  computer in-context?
\newblock {\em Preprint}, 2023.

\bibitem{pfau2024let}
Jacob Pfau, William Merrill, and Samuel~R Bowman.
\newblock Let's think dot by dot: Hidden computation in transformer language
  models.
\newblock {\em arXiv preprint arXiv:2404.15758}, 2024.

\bibitem{radford2019language}
Alec Radford, Jeffrey Wu, Rewon Child, David Luan, Dario Amodei, Ilya
  Sutskever, et~al.
\newblock Language models are unsupervised multitask learners.
\newblock {\em OpenAI blog}, 1(8):9, 2019.

\bibitem{schmidhuber1997discovering}
J~Schmidhuber.
\newblock Discovering neural nets with low kolmogorov complexity and high
  generalization capability.
\newblock {\em Neural Networks}, 10(5):857--873, 1997.

\bibitem{schmidhuber2013powerplay}
J~Schmidhuber.
\newblock Powerplay: Training an increasingly general problem solver by
  continually searching for the simplest still unsolvable problem.
\newblock {\em arXiv preprint arXiv:1112.5309}, 2013.

\bibitem{schmidhuber1997shifting}
J~Schmidhuber, J~Zhao, and M~Wiering.
\newblock Shifting inductive bias with success-story algorithm, adaptive levin
  search, and incremental self-improvement.
\newblock {\em Machine Learning}, 28:105--130, 1997.

\bibitem{schmidhuber2002speed}
Jurgen Schmidhuber.
\newblock The speed prior: a new simplicity measure yielding near-optimal
  computable predictions.
\newblock In {\em International conference on computational learning theory},
  pages 216--228. Springer, 2002.

\bibitem{schmidhuber2015learning}
Jurgen Schmidhuber.
\newblock On learning to think: Algorithmic information theory for novel
  combinations of reinforcement learning controllers and recurrent neural world
  models.
\newblock {\em arXiv preprint arXiv:1511.09249}, 2015.

\bibitem{schuurmans2023memory}
Dale Schuurmans.
\newblock Memory augmented large language models are computationally universal.
\newblock {\em arXiv preprint arXiv:2301.04589}, 2023.

\bibitem{schuurmans2024autoregressive}
Dale Schuurmans, Hanjun Dai, and Francesco Zanini.
\newblock Autoregressive large language models are computationally universal.
\newblock {\em arXiv preprint arXiv:2410.03170}, 2024.

\bibitem{scully2025gittins}
Ziv Scully and Alexander Terenin.
\newblock The gittins index: A design principle for decision-making under
  uncertainty.
\newblock {\em arXiv preprint arXiv:2506.10872}, 2025.

\bibitem{snell2024scaling}
Charlie Snell, Jaehoon Lee, Kelvin Xu, and Aviral Kumar.
\newblock Scaling llm test-time compute optimally can be more effective than
  scaling model parameters.
\newblock {\em arXiv preprint arXiv:2408.03314}, 2024.

\bibitem{solomonoff1999two}
Ray Solomonoff.
\newblock Two kinds of probabilistic induction.
\newblock {\em The Computer Journal}, 42(4):256--259, 1999.

\bibitem{solomonoff1964formal1}
Ray~J Solomonoff.
\newblock A formal theory of inductive inference. part i.
\newblock {\em Information and control}, 7(1):1--22, 1964.

\bibitem{solomonoff1964formal2}
Ray~J Solomonoff.
\newblock A formal theory of inductive inference. part ii.
\newblock {\em Information and control}, 7(2):224--254, 1964.

\bibitem{solomonoff1984optimum}
Raymond~J Solomonoff.
\newblock Optimum sequential search.
\newblock {\em Memorandum, Oxbridge Research, Cambridge, Mass}, 1984.

\bibitem{pmlr-v125-steinke20a}
Thomas Steinke and Lydia Zakynthinou.
\newblock {R}easoning {A}bout {G}eneralization via {C}onditional {M}utual
  {I}nformation.
\newblock In Jacob Abernethy and Shivani Agarwal, editors, {\em Proceedings of
  Thirty Third Conference on Learning Theory}, volume 125 of {\em Proceedings
  of Machine Learning Research}, pages 3437--3452. PMLR, 09--12 Jul 2020.

\bibitem{sutton1999between}
Richard~S Sutton, Doina Precup, and Satinder Singh.
\newblock Between mdps and semi-mdps: A framework for temporal abstraction in
  reinforcement learning.
\newblock {\em Artificial intelligence}, 112(1-2):181--211, 1999.

\bibitem{takahira2016entropy}
Ryosuke Takahira, Kumiko Tanaka-Ishii, and Lukasz Debowski.
\newblock Entropy rate estimates for natural language—a new extrapolation of
  compressed large-scale corpora.
\newblock {\em Entropy}, 18(10):364, 2016.

\bibitem{trager2023linear}
Matthew Trager, Pramuditha Perera, Luca Zancato, Alessandro Achille, Parminder
  Bhatia, and Stefano Soatto.
\newblock Linear spaces of meanings: compositional structures in
  vision-language models.
\newblock In {\em Proceedings of the IEEE/CVF International Conference on
  Computer Vision}, pages 15395--15404, 2023.

\bibitem{vapnik2006transductive}
Vladimir Vapnik.
\newblock {\em Transductive inference and semi-supervised learning}.
\newblock MIT press, 2006.

\bibitem{vapnik2015uniform}
Vladimir~N Vapnik and A~Ya Chervonenkis.
\newblock On the uniform convergence of relative frequencies of events to their
  probabilities.
\newblock In {\em Measures of complexity: festschrift for alexey chervonenkis},
  pages 11--30. Springer, 2015.

\bibitem{weitzman1978optimal}
Martin~L Weitzman.
\newblock Optimal search for the best alternative.
\newblock {\em Econometrica}, 47(3), 1979.

\bibitem{wu2023many}
Jingfeng Wu, Difan Zou, Zixiang Chen, Vladimir Braverman, Quanquan Gu, and
  Peter~L Bartlett.
\newblock How many pretraining tasks are needed for in-context learning of
  linear regression?
\newblock {\em arXiv preprint arXiv:2310.08391}, 2023.

\bibitem{xu2017information}
Aolin Xu and Maxim Raginsky.
\newblock Information-theoretic analysis of generalization capability of
  learning algorithms.
\newblock {\em Advances in neural information processing systems}, 30, 2017.

\bibitem{yu2025dapo}
Qiying Yu, Zheng Zhang, Ruofei Zhu, Yufeng Yuan, Xiaochen Zuo, Yu~Yue, Weinan
  Dai, Tiantian Fan, Gaohong Liu, Lingjun Liu, et~al.
\newblock Dapo: An open-source llm reinforcement learning system at scale.
\newblock {\em arXiv preprint arXiv:2503.14476}, 2025.

\bibitem{zancato2024b}
Luca Zancato, Arjun Seshadri, Yonatan Dukler, Aditya~Sharad Golatkar, Yantao
  Shen, Benjamin Bowman, Matthew Trager, Alessandro Achille, and Stefano
  Soatto.
\newblock B'mojo: Hybrid state space realizations of foundation models with
  eidetic and fading memory.
\newblock {\em Advances in Neural Information Processing Systems},
  37:130433--130462, 2024.

\bibitem{JMLR:v25:23-1042}
Ruiqi Zhang, Spencer Frei, and Peter~L. Bartlett.
\newblock Trained transformers learn linear models in-context.
\newblock {\em Journal of Machine Learning Research}, 25(49):1--55, 2024.

\end{thebibliography}

\appendix

\section{System-theoretic notation}
\label{app:notation}

Since LLMs are stochastic dynamical systems, in this section we write the definition of proper time in notation more familiar to those with system theory background, where a discrete-time, discrete-state stochastic dynamical model is expressed in the form
\begin{equation}
    \begin{cases}
    x_{t+1} = f(x_t)+ g(x_t,u_t) + v_t \\  y_t = h(x_t) + n_t
\end{cases}
\end{equation}
where $x$ is the state, $f$ the drift vector field, $g$ the input map, $v$ the driving noise, $h$ the output map, and $n$ the  measurement noise. The following table maps these terms to the notation in the paper:
\begin{table}[h!]
    \centering
        \caption{Summary of notation.}
    \label{tab:notation}
\begin{tabular}{rcc}
\toprule
 & \textbf{Automata} & \textbf{System Theory}  \\ 
\midrule
state & $s$  & $x$ \\
state transition & $\nu(\cdot|\cdot)$ & $p(\cdot | \cdot)$  \\
state trajectory &$h = (x_1, \dots, s_n)$ & $x_{1:t}$\\
output map & $a = {\rm dec}(s)$ & $y = h(x)$ \\
input map & $s = {\rm enc}(u)$ & $g(u)$ \\
output & $h_t \terminates a$ & $\Phi(x_0,u_{1:t}) = y_{1:t}$ \\
\bottomrule
\end{tabular}

\end{table}

Let $x_{1:t}$ a state trajectory starting from $x_0$ driven by input $u_{1:t}$, which may include a stochastic component. The transition probability is $p(x_{i+1} | x_i, u_i)$ for all $i = 1, \dots, t$. We denote with $x_t = \Phi(x_0, u_{1:t})$ the {\em flow} of the system from $x_0$ to $x_t$ under input $u_{1:t}$, and $P(x_t | x_0, u_{1:t}) = \prod_{i = 1}^t p(x_{i+1}|x_i, u_i)$ the corresponding probability. A target state $x_t$ is {\em reachable} from $x_0$ if there exits at least one such input. We denote by $R_t(x_0) = \{ x_t \ {\rm s. \ t.}  \ \exists  \ u_{1:t} \ | \ x_t = \Phi(x_0, u_{1:t}) \}$ the set of states reachable in $t$ steps, and $R(x_0) = \{R_t(x_0), \ t \in {\mathbb N}\}$ the set of states reachable from $x_0$. Finally, the reachable set $R = \{R(x_0), \ | \ x_0 \in X\}$ is the set reachable from some initial state. Proper time is then defined as
\[
\tau(x_a, x_b) = \min_{t \ | x_b \in R_t(x_a)}  \frac{t}{P(x_b|x_a, u)}
\]
where $u$ is any $t$-long input such that $x_b = \Phi(x_a, u)$. Alternatively, we can write
\[
\tau(x_a, x_b) = \min_{u \ | \ x_b = \Phi(x_a, u)}  \frac{|u|}{P(x_b|x_a, u)}
\]

The reachable set can be written as the indicator function of the verifier $R(x) = \chi_{\{ y \ | \ f(x,y) = 1\}}$.

\end{document}